\documentclass{article}

\usepackage[applemac]{inputenc}

\usepackage{xcolor}
\usepackage{graphicx,epsfig}
\usepackage{amsmath,mathrsfs} 
\usepackage{amssymb} 
\usepackage{amsfonts}
\usepackage{mathtools}
\usepackage{amsthm}

\usepackage{epstopdf}
\usepackage{ifthen}
\usepackage{bbm}
\graphicspath{{eps/}}
\usepackage{hyperref}
\usepackage{booktabs}

\newcommand{\inR}{\in \mathbb{R}}

\newcommand{\R}{ \mathbb{R}}

\newcommand{\N}{ \mathbb{N}}

\newcommand{\eqdef}{\stackrel{\vartriangle}{=}}

\newcommand{\Lop}{{\rm L}}
\newcommand{\Dop}{{\rm D}}

\newcommand{\dint}{{\rm d}}

\DeclareMathOperator*{\esssup}{ess\,sup}
\DeclareMathOperator*{\essinf}{ess\,inf}

\newcommand{\Top}{\mathsf{T}}

\def\V#1{{\boldsymbol{#1}}}         
\def\Spc#1{{\mathcal{#1}}}  
\def\M#1{{\bf{#1}}}  
\def\Op#1{{\mathrm{#1}}}  
\def\ee{\mathrm{e}} 
 
\def\Indic{\mathbbm{1}}

\def\Identity{\mathrm{Id}} %

\renewcommand{\[}{\begin{equation}}
\renewcommand{\]}[1]{\label{eq:#1}\end{equation}}

\newtheorem{definition}{Definition}
\newtheorem{proposition}{Proposition}

\newtheorem{theorem}{Theorem}

\providecommand{\rev}[1]{\textcolor{black}{#1}}

\title{
Controlled Learning of Pointwise Nonlinearities\\ in Neural-Network-Like Architectures\thanks{The research leading to these results was funded in part by the Swiss National Science Foundation under Grant 200020\_219356 and the European Research Council under Grant ERC-2020-AdG FunLearn-101020573.}
}

\author{
Michael Unser\thanks{Biomedical Imaging Group, \'Ecole polytechnique f\'ed\'erale de Lausanne (EPFL),
Station 17, CH-1015, Lausanne, Switzerland ({\tt michael.unser@epfl.ch}). },
Alexis Goujon, Stanislas Ducotterd
 }

\begin{document}



\maketitle
\begin{abstract}
We present a general variational framework for the training of freeform nonlinearities in layered computational architectures subject to some slope constraints. The regularization that we add to the traditional training loss penalizes the second-order total variation of each trainable activation.
The slope constraints allow us to impose properties such as 1-Lipschitz stability, firm non-expansiveness, and monotonicity/invertibility. These properties are crucial to ensure the proper functioning of certain classes of signal-processing algorithms (e.g., plug-and-play schemes, unrolled proximal gradient, invertible flows).  We prove that the global optimum of the stated constrained-optimization problem is achieved with nonlinearities that are adaptive nonuniform linear splines. We then show how to  solve the resulting function-optimization problem numerically by representing the nonlinearities in a suitable (nonuniform) B-spline basis. Finally, we illustrate the use of our framework with the data-driven design of (weakly) convex regularizers for the denoising of images and the resolution of inverse problems.
\end{abstract}


\section{Introduction}
Modern signal/image processing heavily relies on two basic types of computational modules: (i) linear transforms (examples include convolutions, filterbanks, wavelet transforms, and any linear layer of a neural network); and (ii) pointwise nonlinearities, which are typically shared across signal components.

In traditional signal processing, these modules are fixed and justified by mathematical principles \cite{Mallat2009,Vetterli2014} such as the decoupling of the signal (e.g., Karhunen-Loève transform, independent-component analysis) or its efficient encoding (e.g., DCT or wavelets) with a minimal number of atoms (sparsity) \cite{Mallat2009,Bruckstein2009,Elad2010b,Baraniuk2010}. The encoding usually involves some form of thresholding \cite{Donoho:1995,Moulin1999,Chang2000,Kalifa2003,Cosentino2020}, which accounts for the nonlinear part of the processing. 
The building blocks of iterative reconstruction algorithms such as ISTA \cite{Figueiredo2003}, FISTA \cite{Beck2009b}, and ADMM \cite{Ramani2011} for the recovery of signals under sparsity constraints---as in the context of compressed sensing \cite{Donoho2006, Candes2007}---also align with these categories. These algorithms repeatedly alternate between linear steps (e.g., backprojection followed by signal expansion) and a pointwise nonlinearity (e.g., soft-thresholding) until convergence \cite{Figueiredo2007}.

With the rise of machine learning, neural networks are being increasingly integrated into signal-processing algorithms, often with substantial performance benefits \cite{Gregor2010,Chen2014,
Hammernik2018,AgHeMaJaco2019,Effland2020,Monga2021}. This is facilitated by the fact that neural networks employ the same fundamental operations as classic signal processing. \rev{One builds these networks} by stacking linear modules (such as the convolutional layers of the network) and (pointwise) nonlinearities known as activations. Their specificity lies in the tunability of the linear components, a.k.a.\ the weights of the neural network, which are optimized numerically for best performance. This optimization is achieved through a training phase that necessitates access to a large set of representative data \cite{LeCun2015}. 

While researchers have invested a considerable effort in the fine-tuning of the linear components of neural networks, they have devoted much less attention to the exploration of neuronal activation functions. In fact, those are typically kept fixed, in the form of standardized functions such as the rectified linear unit (ReLU) or various flavors of sigmoids \cite{Dubey2022}. Although some authors have strived to adjust parametric nonlinearities \cite{Agostinelli2015,He2015,Chen2017,Apicella2021}, we contend that there remains untapped potential in the training of freeform activations, which presents both conceptual and computational challenges.
%
%
%
%
%

As argued in Section \ref{Sec:DeepExtension}, the learning of a pointwise nonlinearity in any given layered computational architecture can be formally reduced to 
the determination of a continuous function $f: \R \to \R$ such $f(x_m)= z_m$ 
for an appropriate set of points $(x_m,z_m)\in \R^2, m=1,\dots,M$. Without additional assumptions, this problem is ill-posed because the data are finite while a function has an infinite number of degrees of freedom. The common approach is to favor ``regular" functions by the introduction of a roughness penalty (e.g., the energy of some derivative of $f$)  and to seek the solution that minimizes this penalty. For instance, it is well-known that the best data fit that minimizes $\|f'\|^2_{L_2}=\int_\R |f'(x)|^2 \dint x$ (resp., $\|f''\|^2_{L_2}$)
is a nonuniform linear spline (resp. a cubic spline) with knots at the data locations $x_m$ \cite{deBoor1966, Prenter:1975}. While this result is mathematically elegant, it is not very practical because the resulting $f$ has as many knots/parameters as there are data points to be fitted. An attractive alternative is to replace the traditional Hilbertian penalty with ${\rm TV}^{(2)}(f)$ (the second-order total variation of $f$), which has the remarkable property of also yielding linear spline solutions, albeit with a much smaller number of adaptive knots \cite{Mammen1997, Unser2017, Debarre2020}. 
Below, we highlight the distinctive features of ${\rm TV}^{(2)}(f)$ which, in our view \cite{Unser2019c}, make it  the ideal regularizer for our purpose.
\begin{enumerate}
\item It does not penalize linear/affine solutions since these are in the null space of the underlying regularization operator 
(second-order derivative).
\item The condition ${\rm TV}^{(2)}(f)<\infty$ implies that $f$ is differentiable almost everywhere, which is a prerequisite of the celebrated backpropagation algorithm.
\item It privileges simple piecewise-linear solutions with a minimal number of knots (breakpoints) \cite{Debarre2020}. In that respect, we note that the two most popular nonlinearities used in applications---namely, the ReLU activation and the soft-threshold---are linear splines with as few as one and two knots, respectively.  
\item Despite the fact that the problem of fitting a nonuniform parametric linear spline to data is non-convex (because the positions of the knots must also be optimized), the scheme admits a very efficient gridded implementation with the help of uniform B-splines \cite{Bohra2020b}.
\end{enumerate}

Our present contribution---the ``controlled" part of the story---is to refine the framework in order to handle additional inequality constraints on the derivative of $f$ (see Theorem \ref{Theo:SplineFitConstrained}). This extension 
is significant as it enables the optimal design of ``stable'' nonlinearities with a Lipschitz constant of $1$ (such as ReLU), increasing maps, as well as firmly non-expansive nonlinearities that qualify as proximal operators of a convex potential \cite{Nguyen2018,Gribonval2020}. These conditions turn out to be crucial for the robustness and convergence of iterative algorithms, either of the proximal gradient type (ISTA, FISTA) \cite{Beck2009b, Combettes2011,Parikh_proximal_2014,Mardani2018}, or of the plug-and-play type (which requires the non-expansiveness of the denoising step) \cite{Venkatakrishnan2013plug,Chan2016plug,Ryu2019plug,Kamilov2023plug}.

The paper is organized as follows. We establish the notation in Section \ref{Sec:Prelim} and recall some basic results on the continuity and differentiability of functions. 
Section \ref{Sec:TV2Theo} contains the proof of our key result (Theorem \ref{Theo:SplineFitConstrained}), which establishes the optimality of adaptive linear splines for the fitting of data subject to slope constraints.
In Section \ref{Sec:Potentals}, we relate our optimality result to variational signal processing by identifying the conditions under which a learned spline nonlinearity is either the derivative or the proximal operator of a (weakly) convex potential. We also describe a regularization mechanism (Proposition \ref{Prop:reweightedProx}) to adjust a learned proximal map to changes in noise levels. In Section \ref{Sec:Framework}, we apply our theoretical results to the training of freeform activations in deep neural networks and/or in unrolled architectures. We then present a discretization mechanism that extends our prior deep-spline framework \cite{Bohra2020b} in two respects: (i) the use of a more general parameterization of the nonlinearities involving nonuniform B-splines; and (ii) the ability to directly control their slope excursion.
Finally, in Section \ref{Sec:Denoising}, we demonstrate the use of our framework to learn interpretable (weakly) convex potentials via a basic image-denoising task.

\section{Mathematical Preliminaries}
\label{Sec:Prelim}
Let $f: \R \to \R$ be a function that satisfies 
\begin{align}
\label{Eq:LipInequal}
\left|f(y)-f(x)\right| \le L_0 |y-x|
\end{align}
for all $x,y \in \R$ and some constant $L_0$. Such a function is said to be {\em Lipschitz-continuous}. The smallest
constant $L_0$ such that \eqref{Eq:LipInequal} holds is the {\em Lipschitz constant} of $f$, which is denoted by $\|f\|_{\rm Lip}$.
The collection of all functions with a finite Lipschitz constant is denoted by ${\rm Lip}(\R)$. 

Lipschitz continuity is a strong form of (uniform) continuity. In fact, all the members of ${\rm Lip}(\R)$ are absolutely continuous and, therefore, 
differentiable almost everywhere with a measurable and essentially bounded derivative (Rade\-{macher}'s theorem). The Lipschitz constant of the function then corresponds to the essential supremum of its derivative, so that
\begin{align}
\|f\|_{\rm Lip}=\|f'\|_{L_\infty}\eqdef\esssup_{x \inR}  |f'(x)|
\end{align}
where $f'$ is the derivative of $f$.
Conversely, if $f: \R \to \R$ is absolutely continuous with $|f'(x)|<K\  \mathrm{a.e.}$, then $f \in {\rm Lip}(\R)$. Along the same lines, we have that
$\int_a^b f'(x) \dint x= f(b)-f(a)$ for all $f\in { \rm Lip}(\R)$. Finally, we can equip ${ \rm Lip}(\R)$ with the norm
$\|f\|=\|f'\|_{L_\infty}+|\langle \phi,f \rangle|$ and $\phi(x)=(2 \pi)^{-1/2} \ee^{-|x|^2/2}$ (the relevant property here is
$\int_\R \phi(x)\dint x=\langle \phi, 1\rangle=1$
), which then turns it into a Sobolev-type Banach space.

Another useful seminorm is the second-order total variation of $f$ defined as
\begin{align}
{\rm TV}^{(2)}( f )=\|f''\|_{\Spc M} \eqdef \sup_{\varphi \in \Spc S(\R): \|\varphi\|_{L_\infty}\le 1}  \langle f'', \varphi \rangle,
\end{align}
where $f'' \in \Spc S'(\R)$ is the second derivative of $f$ in the sense of distributions and $\Spc S(\R)$ is Schwartz' space of smooth and rapidly decreasing test functions. The space of functions
with bounded second-order variation is denoted by ${\rm BV}^{(2)}(\R)$. Similarly to ${\rm Lip}(\R)$, we can equip
${\rm BV}^{(2)}(\R)$ with the norm $\|f\|_{{\rm BV}^{(2)}}={\rm TV}^{(2)}( f )+\sqrt{|\langle \phi,f \rangle|^2+|\langle \phi',f \rangle|^2}$, where the role of the second term\footnote{The guiding principle in the selection of the linear functionals $\phi$ and $(-\phi')$ is their biorthogonality
with a basis of the null space of $\frac{\dint^2}{\dint x^2}$; more precisely, the conditions $\langle \phi,x\rangle=0$ (from the symmetry of $\phi$), $\langle -\phi',1\rangle=\langle \phi,0\rangle=0$, and $\langle -\phi',x\rangle=\langle \phi,1\rangle=1$ (integration by part), which leaves us a wide range of possibilities.} is to remove the ambiguity for the affine components $x \mapsto b_0 + b_1 x,\, (b_0,b_1) \in \R^2$ that span the null space of the second-derivative operator  \cite[Appendix B]{Unser2019c}. 

\subsection{Continuity Bounds}
It turns out that the ${\rm TV}^{(2)}$-seminorm is stricter than the Lipschitz one, which implies that ${\rm BV}^{(2)}(\R)$ is continuously embedded in
${\rm Lip}(\R)$.

\begin{theorem}[\cite{Aziznejad2022}]
Any function with finite second-order total variation is Lipschitz-continuous with its
Lipschitz constant  being bounded by 
\begin{align}
\label{Eq:LipInequal2}
\|f\|_{\rm Lip}\le {\rm TV}^{(2)}( f ) + \ell_{\inf}(f),
\end{align}
where 
\begin{align}
\ell_{\inf}(f)=\inf_{x \ne y}
\frac{\left|f(y)-f(y)\right|}{ |y-x|}=\essinf_{x \inR}  |f'(x)|.
\end{align}
Moreover, \eqref{Eq:LipInequal2} is saturated if and only if $f$ is monotone-convex or monotone-concave.
\end{theorem}

The range of the derivative of $f\in { \rm Lip}(\R)$ is characterized  by the two constants \begin{align}
s_{\min}(f) &= \essinf_{x \inR}  f'(x)\\
s_{\max}(f) &= \esssup_{x \inR}  f'(x).
\end{align}
The Lipschitz continuity of $f$ allows us to state the general slope inequality
\begin{align}
\label{Eq:Slopeinequal}
s_{\min}(f) 
\le \frac{f(y)-f(x)}{y-x} \le s_{\max}(f)
\end{align}
for any $x,y \in \R$ with $x< y$. In fact, since $\frac{f(y)-f(x)}{y-x}=\frac{f(x)-f(y)}{x-y}$,
\eqref{Eq:Slopeinequal} remains valid for any $x,y \in \R$ with $x\ne y$. We note that the lower and upper  bounds in \eqref{Eq:Slopeinequal} are tight and that $\|f\|_{\mathrm{Lip}}=\max(|s_{\max}(f)|,|s_{\min}(f)|)$ where $s_{\max}(f)$
and $s_{\min}(f)$ can be interpreted as the maximal and minimal slope of $f$, respectively.

\subsection{Canonical Interpolation of an Ordered Set of Points}
\label{Sec:Canonical}
In the sequel, we shall use the symbol $\mathbb{P}=\big\{(x_n,f_n)\big\}_{n=
1}^N$ to denote a generic set of data points on the real line
with $-\infty<x_1< x_2< \dots< x_N<+\infty$ and $f_1, \dots,f_N \inR$.
It is also convenient to identify the geometric slopes of  $\mathbb{P}$ as
\begin{align}
\label{Eq:slopes}
s_n=s_n(\mathbb{P})= \frac{ f_n-f_{n-1}}{x_n-x_{n-1}}, \quad n=2, \dots, N,
\end{align}
and the corresponding bounding constants
\begin{align}
s_{\min}(\mathbb{P})&=\min \left(s_n\right)_{n=2}^N\\
s_{\max}(\mathbb{P})&= \max \left(s_n \right)_{n=2}^N.
\end{align}
As preliminary step, we consider the interval $\Omega_n = [x_{n-1},x_n]$ and investigate the search for
a continuous function that optimally interpolates the boundary points in the sense that its slope has the tightest range. The
optimization is performed over the set of admissible interpolators
\begin{align}
\label{Eq:2pointInt}
\mathbb{I}_{n-1,n}=\{f \in  {\rm Lip}(\R): f(x_{n-1})=f_{n-1} \mbox{ and } f(x_{n})=f_n\}.
\end{align}
By setting $s_{\min}(f)=s_{\max}(f)=C$ in \eqref{Eq:Slopeinequal}, we find that the optimal solution is such that
$f'(x)=C\  \mathrm{a.e.}$ in $\Omega_n$ which, when combined with the two interpolation constraints, yields the solution
$f^\ast(x)=f_{n-1} + C(x-x_{n-1})$ with $C=s_n$. 
As for any other $f \in \mathbb{I}_{n-1,n}$, we always have that $s_{\min}(f)\le s_n \le s_{\max}(f)$. Morever, when $f$ is differentiable over $\Omega_n$, 
there necessarily exists a point $c \in \Omega_n$ 
such that $f'(c)=s_n$ (by the mean value theorem). 
This shows that the linear interpolator has the {\em tightest slope excursion} as well as the smallest Lipschitz constant (min-Lip problem) {\em among all admissible interpolators}.  

The argument readily extends to the complete set $\mathbb{P}$ of points.
Indeed, for any interpolator $f_{\rm int} \in {\rm Lip}(\R)$ such that $f_{\rm int}(x_n)=f_n, n=1,\dots,N$, we have that
\begin{align}
s_{\min}(f_{\rm int})=\inf_{x,y \inR:\; x\ne y} \frac{f(y)-f(x)}{y-x} \le s_{\min}(\mathbb{P})\\
s_{\max}(f_{\rm int})=\sup_{x,y \inR:\; x\ne y} \frac{f(y)-f(x)}{y-x} \ge s_{\max}(\mathbb{P})
\end{align}
because $\mathbb{P}$ is a subset of $\R\times \R$ with these two bounds being tight for
the canonical linear-spline interpolator (see Definition \ref{Def:Canonical} below). 
The caveat, however, is that the solution to the tight-slope problem (resp., the min-Lip problem) is no longer unique, unless the points are colinear.

\begin{definition}
[Canonical interpolator]
\label{Def:Canonical}
The canonical interpolator $f_{\mathrm{int}, \mathbb{P}} : \R \to \R$ of $\mathbb{P}$  is the unique continuous piecewise-linear (CPWL) function that interpolates $\mathbb{P}$ and is differentiable over $\R\backslash \{x_2 , \dots , x_{N-1}\}$.
\end{definition}
\begin{figure}
\center
 \includegraphics[width = 10cm]{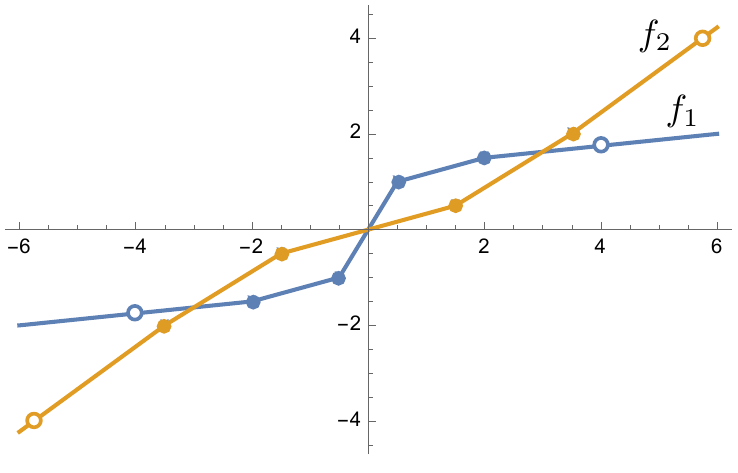}
 \caption{\label{Fig: SplineInterpol} \rev{Canonical spline interpolants for two sets $\mathbb{P}_1$ and $\mathbb{P}_2$ of points  represented as small circles in the plane. 
The filled circles are the spline knots (breakpoints), while the empty ones are the boundary points used for linear extrapolation. The two splines are linked because they are induced by a common (learnable) convex potential $\phi$ with $f_1=\phi'$ and $f_2={\rm prox}_{\phi}$. (See detailed explanations Section 4.)}}
\end{figure}

In other words, $f_{\mathrm{int}, \mathbb{P}}$ is the piecewise-linear spline with knots (a.k.a.\ breakpoints) at $x_2, \dots, x_{N-1}$ that satisfies the interpolation conditions $f_{\mathrm{int}}(x_n)=f_n, n=1,\dots, N$ and that extends linearly
beyond the interval $[x_1,\dots, x_{N}]$ or, equivalently, \rev{fulfills} natural boundary conditions at $x_1$ and $x_{N}$, \rev{as illustrated in Figure \ref{Fig: SplineInterpol}}. In general, $f_{\mathrm{int}, \mathbb{P}}$ is composed of $(N-1)$ linear segments and its derivative is piecewise-constant with
\begin{align}
f'_{\mathrm{int}, \mathbb{P}}(x)=\begin{cases}
s_2,& x < x_1\\
s_n,& x \in [x_{n-1}, x_n), n\in \{2,\dots,N\}\\
s_{N}, & x \ge x_N
\end{cases}
\end{align}
with $s_{\min}(\mathbb{P})\le f'_{\mathrm{int}, \mathbb{P}}(x) \le  s_{\max}(\mathbb{P}).$
Also relevant to our investigation is the observation that the second-order total variation
of the canonical interpolant is
\begin{align}
\label{Eq:TV2spline}
{\rm TV}^{(2)}(f_{\mathrm{int}, \mathbb{P}})={\rm TV}^{(2)}(\mathbb{P})= \sum_{n=3}^N |s_n-s_{n-1}|,
\end{align}
 while its Lipschitz constant is simply 
 \begin{align}
 {\rm Lip}(f_{\mathrm{int}, \mathbb{P}})={\rm Lip}(\mathbb{P})= \max(|s_{\max}(\mathbb{P})|,|s_{\min}(\mathbb{P})|).
 \end{align}
 The conclusion of this section is that there is no interpolator of $\mathbb{P}$ in ${\rm Lip}(\R)$ that achieves
a Lipschitz constant smaller than ${\rm Lip}(\mathbb{P})$ or/and such that the range of its slope is tighter than 
$[s_{\min}(\mathbb{P}), s_{\max}(\mathbb{P})]$. It is also known that the same holds true for the second-order total variation of an interpolator, which cannot be smaller than ${\rm TV}^{(2)}(\mathbb{P})$. 

While we have just seen that the solution that is optimal according to any of the mentioned criteria is achieved by the canonical interpolator in Definition \ref{Def:Canonical}, one should not be fooled by the simplicity of this description.
It turns out that this kind of non-Hilbertian functional-minimization problem admits an infinity of solutions\footnote{By contrast, it is well-known that the
minimization of the Hilbertian energy $\|f'\|^2_{L_2}$ results in a unique solution that matches the canonical spline interpolator for $x \in [x_1,x_{N}]$ and that is constant outside the primary interval with $f(x)=f(x_1)$ for $x\le x_1$ and $f(x)=f(x_N)$ for $x\ge x_N$.}, including some adaptive piecewise-linear splines that have fewer knots
than the canonical interpolator, the non-intuitive part being that these knots do not necessarily coincide with the abscissa of the data points. In the case of the minimization of ${\rm TV}^{(2)}(f)$, the sparsest spline solution is essentially unique and can be determined using the Debarre algorithm \cite[Theorem 2]{Debarre2020}.

\section{Representer Theorem for Constrained ${\rm TV}^{(2)}$ Minimization}
\label{Sec:TV2Theo}

We now present the theorem that provides the theoretical foundation for this paper.
It is an extension/unification of two of our earlier results \cite{Debarre2020, Aziznejad2022}. 
\rev{While Theorem \ref{Theo:SplineFitConstrained} is stated in the context of a generic 1-dimensional data-fitting problem,
we shall see in the second part of the paper how this theoretical result on the optimality of splines is applicable to the training of neuronal activations 
in deep neural networks (Section \ref{Sec:Framework}) and to the data-driven design of {(weakly-)}convex regularizers and proximal operators for image reconstruction (Section \ref{Sec:Denoising}). }

\begin{theorem}
\label{Theo:SplineFitConstrained}
Let us consider the following setting.
\begin{itemize}
\item A strictly convex and coercive function $E : \R \times \R \to \R$.
\item A series of data points $(x_m,y_m) \in \R \times \R$ with $m=1,\dots,M$ and
$x_1< \dots < x_m < x_M$.
\item An adjustable regularization parameter $\lambda \in \R^+$.
\item Two adjustable slope-excursion parameters  $s_{\min} < s_{\max} \in \R$.
\end{itemize}
Then, the solution set of the constrained functional optimization problem in
\begin{align}
\label{Eq:OptProblem}
S= \arg \min_{f \in {\rm BV}^{(2)}(\R)} \left(\sum_{m=1}^M E\big(f(x_m), y_m\big) + \lambda \|f''\|_{\Spc M} \right)\nonumber\\
\hspace*{3cm}\mbox{ s.t. }  s_{\min} \le f'(x) \le s_{\max} \mbox{ a.e.}
\end{align}
is a nonempty, convex, and weak*-compact subset of ${\rm BV}^{(2)}(\R)$ whose extreme points are 
piecewise-linear splines 
with no more than $(M-1)$ linear regions. 

Moreover, there exists a unique vector
$\M z=(z_m)\in \R^M$ such that
\begin{align}
\label{Eq:IntpolTV2}
S= \arg \min_{f \in {\rm BV}^{(2)}(\R)} \|\Dop^2 f\|_{\Spc M} \quad \mbox{ s.t. } \quad  f(x_m)=z_m,\,  m=1,\dots, M,
\end{align}
where the latter reformulation absorbs the initial slope constraints.
\end{theorem}
Since $f \in {\rm BV}^{(2)}(\R)$ is Lipschitz-continuous, we can also formulate the constraint on the derivative as
\begin{align}
\label{Eq:InequalSlope}
s_{\rm min}\; (x_2- x_1)\le f(x_2)-f(x_1) \le s_{\rm max}\; (x_2- x_1)
\end{align}
for all $x_2,x_1 \in \R$ with $x_2>x_1$, without loss of generality. This form is more tractable mathematically because it holds everywhere on $\R$ (as opposed to the {\em almost everywhere} statement on the derivative of $f$).

\begin{proof}
To prove existence, we reformulate the problem as an unconstrained optimization
by augmenting the cost with the barrier functional $i_{C_{\rm slope}}$
where \rev{the set of constraints} is
\begin{align}
C_{\rm slope}=\{f \in {\rm BV}^{(2)}(\R) \mbox{ subject to  } \eqref{Eq:InequalSlope}\}
\end{align}
and 
\begin{align}
i_{C}\eqdef \begin{cases}
0,& \mbox{if } f \in C\\
+\infty,&\mbox{ otherwise.}
\end{cases}
\end{align}
It then suffices to show that the augmented functional $$J_{\rm slope}(f)=\sum_{m=1}^M E\big(f(x_m), y_m\big) + \lambda \|\Dop^2 f\|_{\Spc M} + i_{C_{\rm slope}}$$ is coercive and lower-semicontinuous in the weak* topology
of ${\rm BV}^{(2)}(\R)$. We already know from previous work that $J(f)=\sum_{m=1}^M E\big(f(x_m), y_m\big) + \lambda \|\Dop^2 f\|_{\Spc M}$ is coercive and lower-semicontinuous (see \cite[proof of Theorem 4 with $\Lop=\Dop^2$]{Gupta2018}). The fact that $i_{C_{\rm slope}}$ is non-negative directly implies that $J_{\rm slope}(f)$ is coercive as well.
The only missing ingredient is the lower semicontinuity of $i_{C_{\rm slope}}$, which is automatically met if the constraint box $C_{\rm slope}$ is closed.

To prove that $C_{\rm slope}$ is a weak*-closed convex subset of ${\rm BV}^{(2)}(\R)$, we now consider some sequence $(f_n)_{n \in \N}$ of functions  in $C_{\rm slope}$ that converge to
$f_{\rm lim}$ in the weak* topology. 
For any $n\in \N$ and $x_2 > x_1 \inR$, we have that
\begin{align*}
f_{\lim} (x_2)-f_{\lim}(x_1)-\underbrace{\big(f_{\lim} (x_2)-f_{n}(x_2)\big)}_{\epsilon_n(x_2)}+ \underbrace{\big(f_{\lim} (x_1)-f_{n}(x_1)\big)}_{\epsilon_n(x_1)}
= f_n(x_2)-f_n(x_1)
\end{align*}
which, in view of \eqref{Eq:InequalSlope}, yields the inequality
\begin{align}
\label{Eq:Inequaln}
s_{\min}\;(x_2-x_1)\le f_{\lim} (x_2)-f_{\lim}(x_1)-\epsilon_n(x_2) +\epsilon_n(x_1)\le  s_{\max}\;(x_2-x_1).
\end{align}
Since the sampling functional  $\delta(\cdot - x_m): f \mapsto f(x_m)$ is weak*-continuous in ${\rm BV}^{(2)}(\R)$ for any $x_m \in \R$ (see \cite[Theorem 1, Item 2]{Unser2019c}),
we have that $f_n(x_2)\to f_{\lim}(x_2)$ and $f_n(x_1)\to f_{\lim}(x_1)$  as $n \to \infty$ , which is equivalent to $\lim_{n \to \infty}\epsilon_n(x_2)=0$ and $\lim_{n \to \infty}\epsilon_n(x_1)=0$. The desired bound is the limit form of 
\eqref{Eq:Inequaln} as $n\to \infty$, which ensures that $f_{\lim} \in C_{\rm slope}$ (closedness property).

Since our problem admits at least one minimizer and since the data term in \eqref{Eq:OptProblem} is strictly convex, we can use a standard argument in convex analysis to show that there exists a unique
$\M z \inR^M$ such that $f^\ast(x_m)=z_m$ for all $f^\ast \in S$. 
This allows us to rewrite \eqref{Eq:OptProblem} as the solution set of the (constrained) interpolation problem
\begin{align}
\label{Eq:IntpolTV2b}
\arg \min_{f \in C_{\rm slope}} \|\Dop^2 f\|_{\Spc M} \quad \mbox{ s.t. } \quad  f(x_m)=z_m,\,  m=1,\dots, M.
\end{align}
Now, the equivalence between
\eqref{Eq:IntpolTV2} and \eqref{Eq:IntpolTV2b} is not obvious because \eqref{Eq:IntpolTV2} involves the much larger search space ${\rm BV}^{(2)}(\R)$ that does not explicitly impose the slope constraint.

The last part of the proof is to show that \eqref{Eq:IntpolTV2}, whose complete solution set has been characterized in \cite{Debarre2020}, implicitly imposes the constraint via the proper adjustment of the vector $\M z$. To that end,
we consider a generic member $f^\ast \in S$ of the solution set 
with the unconstrained problem being parametrized by $\mathbb{P}=\big\{(x_m,z_m)\big\}_{m=1}^M$.
We know from \cite[Theorem 2]{Debarre2020} that $f^\ast(x)$ coincides with the canonical interpolator $f_{{\rm int},\mathbb{P}}(x)$ for $x \notin (x_2, x_{M-1})$. As for each of the remaining intervals $[x_{m}, x_{m+1}]$, there are three possible scenarios: (i) $f^\ast$ follows $f_{{\rm int},\mathbb{P}}$ exactly; (ii) $f^\ast$ is convex over the extended interval $[x_{m-1} , x_{m+2}]$; or (iii) $f^\ast$ is concave over $[x_{m-1} , x_{m+2}]$. Let $m$ be the index of an interval $[x_{m},x_{m+1}]$ over which $f^\ast$ deviates from $f_{\mathbb{P}}$. 
The convexity of $f^\ast$ in Scenario (ii) is equivalent to $\frac{f^\ast(x) -f^\ast(y)}{x-y}$ 
 being monotonically nondecreasing in $x$ for every fixed $y$, or vice versa.
The latter property implies that $s_m\le s_{m+1} \le s_{m+2}$ and
\begin{align}
s_m=\frac{z_m-z_{m-1}}{x_m-x_{m-1}}\le \frac{f^\ast(x) -f^\ast(y)}{x-y} \le \frac{z_{m+2}-z_{m+1}}{x_{m+2}-x_{m+1}}=s_{m+2},
\end{align}
for any $x,y \in [x_{m},x_{m+1}]$ with $x\ne y$.
Likewise, one gets the reverse inequalities when $f^\ast$ is concave. 

It is also possible to state these conditions in terms of derivatives.
\begin{enumerate}
\item If $f^\ast$ is convex over $[x_m,x_{m+1}]$, then its derivative ${f^{\ast}}'$ is 
nondecreasing with $s_m \le {f^{\ast}}'(x) \le s_{m+2}\  \mathrm{a.e.}$

\item If $f^\ast$ is affine (i.e., both convex and concave) over $[x_m,x_{m+1}]$,
then ${f^{\ast}}'(x)=s_{m+1}$.
\item If $f^\ast$ is concave over $[x_m,x_{m+1}]$,
then ${f^{\ast}}'$ is 
non-increasing with $s_{m+2} \le {f^{\ast}}'(x) \le s_{m} \ \mathrm{a.e.}$, where $s_{m+2}\le s_{m+1} \le s_{m}$.
\end{enumerate}
The bottom line is that all the members $f^\ast$ of the solution set, including the canonical interpolator, tightly fulfill the slope inequality $s_{\min}=s_{\min}(\mathbb{P}) \le {f^\ast}'(x) \le s_{\max}=s_{\max}(\mathbb{P})$, where the two constants are now explicitly connected to $\M z$. 
Since there is no function among all possible interpolators that achieves a tighter slope excursion (see Section \ref{Sec:Canonical}), we can drop the slope constraint in the interpolation reformulation of the problem. 
The convexity and weak*-compactness of $S$ and the form of its extreme points then directly follow from  \cite[Theorem 1]{Debarre2020} (see also \cite{Unser2017} with $\Lop=\Dop^2$).
\end{proof}

\section{Scalar Potentials Related to Linear Splines}
\label{Sec:Potentals}

A function $\phi: \R \to \overline{\R}\eqdef\R \cup\{+ \infty\}$ that is proper
, lower-semicontinuous (l.s.c.), and convex (respectively, $\rho$-weakly convex) is called a scalar potential. 
For the precise definition of these properties, the reader is referred to the appendix, which provides a summary of the primary concepts of finite-dimensional convex analysis.

Of special relevance to us is the proximity operator of a $\rho$-weakly convex potential $\phi: \R \to \overline{\R}$ with $0\le \rho<1$, which is defined as
\begin{align}
\label{Eq:Prox}
{\rm prox}_\phi(x)=\arg \min_{z \inR} \left( \tfrac{1}{2}|x-z|^2 + \phi(z)\right).
\end{align}
Let us observe that the functional on the right-hand side of \eqref{Eq:Prox} is coercive, l.s.c., and strictly convex, which guarantees the existence and uniqueness of the minimizer.

In the sequel, we shall investigate two scenarios: (i) the case where $\phi'$ is a (learned) linear spline; and (ii) the case where
$\mathrm{prox}_\phi$ is a linear spline, \rev{within the respective philosophies of \cite{Chen2017} and \cite{Nguyen2018}}. 
The important point is that we can control the convexity properties of $\phi$ by imposing suitable monotonicity/stability constraints on either $\psi=\phi'$ or $f=\mathrm{prox}_\phi$, \rev{as summarized in Proposition \ref{Prop:Convexity}.
The main outcome is that we shall be able to enforce the required conditions within the framework of Theorem \ref{Theo:SplineFitConstrained}.}

\rev{A function $f: \R \to \R$ is said to be monotone (or nondecreasing) if $f(y) \ge f(x)$ for all $y>x \in \R$.
Mathematically, this condition can also be expressed as $\big(f(y)-f(x)\big)(y-x)\ge 0$, in adequacy with Item 1 in Definition \ref{Def:Operatorprops}, Appendix A.3.}
\rev{
\begin{proposition}
\label{Prop:Convexity}
Let $\phi: \R \to \overline{\R}$ be a proper l.s.c.\ scalar potential. Then, we have the following equivalences, with the conditions holding for all $x,y \in {\rm dom}(\phi)=\{x \in \R:-\infty< \phi(x)<+\infty\}$.
\begin{enumerate}
\item $\phi$ is convex: $\phi(\theta x + (1-\theta) y)\le \theta \phi(x) + (1-\theta)\phi(y)$ for any $\theta \in [0,1]$.
\item $R_\phi(x,y)=\frac{\phi(x)-\phi(y)}{x-y}$ is monotone
in $x$ for any fixed $y$ (or vice versa).
\item $f=\mathrm{prox}_\phi \in {\rm Lip}(\R)$ with 
$0\le f'(x)\le 1 \ \mathrm{a.e.}$
%
\item (under the additional assumption that $\phi$ is differentiable)\\ $\phi(y)\ge \phi(x)+\phi'(x) (y-x)$.
\item (under the additional assumption that $\phi$ is differentiable)\\ $\psi=\phi'$ is monotone.
\item (under the additional assumption that $\phi$ is differentiable with $\phi' \in {\rm Lip}(\R)$)\\ 
$\phi''(x)\ge 0 \ \mathrm{a.e.}$
\end{enumerate}
Likewise, if  $\phi' \in {\rm Lip}(\R)$ (as in Item 6), then $\phi$ is $\rho$-weakly convex if and only if $\phi''(x)\ge -\rho \ \mathrm{a.e.}$
\end{proposition}
Items 4 and 6 are the scalar transcriptions of the classic first-order (see \eqref{Eq:firstorder}, Appendix A.1) and second-order conditions of convex optimization  \cite[p 69-71]{Boyd2004convex}. Item 3 is equivalent to $\phi$ being {\em firmly non-expansive} (see Definition \ref{Def:Operatorprops}, Item \ref{Item:Firm} in Appendix A.3), which is the necessary and sufficient condition for a proper l.s.c.\ function to be convex \cite{Nguyen2018,Gribonval2020}. The equivalence in Item 2 is specific to the univariate setting and implies the other ones, as we briefly show below. 
\begin{proof}[Sketch of Proof]
\item $2 \Leftrightarrow 1$: For any fixed $y$, the monotonicity condition in Item 2 can be stated as
\begin{align}
\label{Eq:Convex2}
R_\phi(z,y)=\frac{\phi\big(z\big)-\phi(y)}{z-y}\ge \frac{\phi\big(x\big)-\phi(y)}{x-y} \mbox{ for any } z > x.
\end{align}
To show the equivalence with the standard definition of convexity (Item 1), one needs to consider three distinct configurations.
\item 1) Order $z>x> y$: We set $x=y+\theta_1(z-y)=\theta_1 z +(1-\theta_1)y$ with $\theta_1=\frac{x-y}{z-y}\in(0,1)$ and write
\eqref{Eq:Convex2} as $\theta_1\big(\phi(z)-\phi(y)\big) \ge \phi\big(\theta_1 z + (1-\theta_1)y\big)-\phi(y)$. When $z>y$, the latter is equivalent to the convexity condition in Item 1.
\item 2) Order $y>z> x$: We set $z=y+\theta_2(x-y)=\theta_2 x +(1-\theta_2)y$ with $\theta_2=\frac{z-y}{x-y}\in(0,1)$ and rewrite
\eqref{Eq:Convex2} as $\phi\big(\theta_2 x+(1-\theta_2)y)\big)-\phi(y) \le  \theta_2 \big(\phi(x)-\phi(y)\big)$, which again is equivalent to the condition in Item 1
when $y>x$.
\item 3) Order $z>y> x$: By defining $\theta_3=\frac{y-z}{x-z}
\in(0,1)$, we rewrite
\eqref{Eq:Convex2} as $ (1-\frac{1}{\theta_3})\left(\phi(z)- \phi(y)\right) \ge\phi(x)-\phi(y)$, which is itself equivalent to
$ {\left(\frac{1}{\theta_3}-1\right)}\phi(z) + \phi(x) \ge \frac{1}{\theta_3}\phi(y)$.
We then get the desired convexity relation by renormalization and substitution of $y=\theta_3 x +(1-\theta_3) z$.
\item $2 \Rightarrow 4,5$: 
If $\phi$ is differentiable, then \eqref{Eq:Convex2} is also valid in the limit as $z \to y$, which yields the desired first-order 
condition with $\lim_{z\to y}\frac{\phi(z) - \phi(y)}{z-y}=\phi'(y)$. Likewise, for $x\ge y$, we have
that $\phi'(x)\ge \frac{\phi(x)-\phi(y) }{x-y }\ge \phi'(y)$, which shows that $\phi'$ is monotone.
which indicates that $\phi'$ is monotone.
\item Item 6 and weakly-convex case:  If $\psi=\phi' \in {\rm Lip}(\R)$, then the second derivative $\phi''$ is defined almost everywhere so that $\phi$ is convex and $\phi'$ is monotone if and only if $\phi''(x)\ge 0 \ \mathrm{a.e.}$ Likewise, the $\rho$-weak-convexity of $\phi$ is  equivalent to
the convexity of $\tfrac{\rho}{2} (y-x)^2 + \phi(x)$ (see Item 4, Definition \ref{Ref:Convex}), which yields
$\rho+\phi''(x)\ge 0 \ \mathrm{a.e.}$
\end{proof}}

These considerations lead to the identification of the following configurations of interest for machine learning. 


\begin{definition}
\label{Def:Monotone}
Let $f \in {\rm Lip}(\R)$. Then, the following categorization holds with $\rho\ge0$.
\begin{enumerate}
\item $f$ is non-decreasing (monotone)$  \quad\Leftrightarrow \quad 0\le f'(x)\, \mathrm{a.e.}$
\item $f$ is firmly non-expansive$  \quad\Leftrightarrow \quad0\le f'(x)\le1\, \mathrm{a.e.}$
\item $f$ is 1-Lipschitz$  \quad\Leftrightarrow \quad-1\le f'(x)\le1\, \mathrm{a.e.}$
\item $f$ is $\rho$-strongly increasing (monotone)$  \quad\Leftrightarrow \quad\rho\le f'(x)\, \mathrm{a.e.}$
\item $f$ is $\rho$-weakly increasing $  \quad\Leftrightarrow \quad-\rho\le f'(x)\, \mathrm{a.e.}$
\end{enumerate}
\end{definition}

We also note that if $\phi: \R \to \R$ is symmetric, then ${\rm prox}_\phi(x)$ is anti-symmetric, and that the same holds true for
$\phi'(x)$ within the domain where the derivative is well-defined.

%

\subsection{Scalar Potential Specified Through its Derivative}
The generic form of a
piecewise-linear spline with knots $\tau_1 < \tau_2 < \cdots < \tau_K$ is
\begin{align}
\label{Eq:Linspline}
f_{\rm spline}(x)=b_0+ b_1 x + \sum_{k=1}^K a_k (x-\tau_k)_+,
\end{align}
where $b_0, b_1 \inR$ and $(a_k) \in \R^{K}$ are the linear weights of the model.
One can readily verify that the derivative of $f$ is piecewise-constant with its range being constrained by
\begin{align}
s_{\min}=\inf_{x\in \R} f_{\rm spline}'(x) =\min\{b_1 + \sum^k_{n=1} a_n\}^K_{k=1}\label{Eq:sminspline}\\
s_{\max}=\sup_{x\in \R} f_{\rm spline}'(x)=\max\{b_1 + \sum^k_{n=1} a_n\}^K_{k=1}\label{Eq:smaxspline}.
\end{align}
The other relevant property is that the second-order total variation of a linear spline is given the $\ell_1$-norm of its (ReLU) coefficients as
\begin{align}
\label{Eq:TV2discrete}
\mathrm{TV}^{(2)}(f_{\rm spline})=\|f''_{\rm spline}\|_{\Spc M}=\| \sum_{k=1}^K a_k \delta(\cdot-\tau_k)\|_{\Spc M}=\sum_{k=1}^K |a_k|=\|\M a\|_{\ell_1}.
\end{align}
Since $\ell_1$-norm minimization promotes sparsity, the penalization of $\mathrm{TV}^{(2)}(f_{\rm spline})$
tends to reduce the number of active knots of the spline.

 \begin{proposition}[Spline derivative of a (weakly) convex potential]
\label{Proposition:Derivative}
Consider the generic linear spline 
\eqref{Eq:Linspline}
with knots $\tau_1 < \tau_2 < \cdots < \tau_K$ and expansion coefficients $b_0,b_1, a_1, \dots, a_K\in \R$.
Then, there exists a unique quadratic-spline potential function $\phi: \R \to \R$
such that $f_{\rm spline}(x)=\phi'(x)$ and $\phi(0)=0$ (without loss of generality). Depending on the value of
$s_{\min}$ given by
\eqref{Eq:sminspline},
the potential $\phi$ is endowed with the following properties.
\begin{enumerate}
\item If $s_{\min}\ge0$, then $\phi$ is convex $\quad \Leftrightarrow \quad$ $f_{\rm spline}$ is nondecreasing
\item If $s_{\min}>0$, then $\phi$ is $s_{\min}$-strongly convex $\quad \Leftrightarrow \quad$ $f_{\rm spline}$ is $s_{\min}$-strongly increasing.
\item If $s_{\min}<0$, then $\phi$ is $|s_{\min}|$-weakly convex $\quad \Leftrightarrow \quad$ $f_{\rm spline}$ is $|s_{\min}|$-weakly increasing.
\end{enumerate}
\end{proposition}
\begin{proof}
The potential $\phi(x)$ is found by calculating the primitive of \eqref{Eq:Linspline}, which 
is
\begin{align}
\label{Eq:spline2}
\phi(x)= c_0 + b_0 x + \frac{b_1}{2} x^2 + \sum_{k=1}^K \frac{a_k}{2}(x-\tau_k)_+^2
\end{align}
where $c_0$ is an integration constant that is set to $c_0=-\sum_{k=1}^K\frac{a_k}{2}(-\tau_k)_+^2$ to fulfill the boundary condition $\phi(0)=0$. As it turns out, $\eqref{Eq:spline2}$ is the generic form of a quadratic spline
with knots at the $\tau_k$. Indeed, $\phi(x)$ is differentiable twice with
\begin{align}
\label{Eq:phipp}
\phi''(x)=f'_{\rm spline}(x)=\sum_{k=1}^K a_k\Indic_+(x-\tau_k),
\end{align}
where $\Indic_+(x)=\begin{cases} 1,& x\ge 0\\
0, &x<0\end{cases}$ is the unit-step function, which reveals that $\phi''$ is constant on each subinterval $[\tau_{k-1},\tau_{k})$ for $k=1,\dots,K+1$ with the convention that $\tau_0=-\infty$ and $\tau_{K+1}=+\infty$.
The latter property is consistent with $\phi$ itself being a quadradic polynomial on each of these subintervals.

As for the convexity of $\phi$, we rely on \eqref{Eq:phipp} to obtain the slope inequality
$s_{\min} \le f'_{\rm spline}(x)$ with the infimum given by \eqref{Eq:sminspline}.
The properties then directly follow from the characterizations in Definition \ref{Def:Monotone}.
\end{proof}
\subsection{Scalar Potential Specified Through its Proximity Operator}
To investigate proximal operators that are piecewise-linear maps, it is convenient to represent such maps
by a minimal set of points. For instance, the generic linear spline given by \eqref{Eq:Linspline} has a unique description in terms of its breakpoints $\big(\tau_k, f(\tau_k)\big)$ for $k=1,\dots,K$ plus two ``outside'' points $\big(\tau_0, f(\tau_0)\big), \big(\tau_{K+1}, f(\tau_{K+1})\big)$ 
where the two additional sampling locations $\tau_0<\tau_1$ and $\tau_{K+1}>\tau_{K}$ can be selected arbitrarily. The idea is that, if we know the knots of the spline, then we only need $(K+2)$ linearly independent equality constraints to uniquely determine 
the weights $b_0, b_1, a_1,\dots, a_K \in \R$ in \eqref{Eq:Linspline}.

Points provide an intuitive description of piecewise-linear curves, including those that exhibit discontinuities.
To formalize the concept, let us consider an ordered set of points $\mathbb{P}=\{(x_1,y_1), \dots, (x_N,y_N)\} \subset \R^2$ 
with $x_1 \le x_2 \le  \cdots \le x_N$. Such a set specifies a piecewise-linear function $f_{\mathrm{int},\mathbb{P}}$ whose graph is obtained by connecting all successive pairs of points by a straight line and by extrapolating the two boundary lines toward infinity. 
In this geometric setting, an off-boundary element $(x_n,y_n) \in \mathbb{P}$ is called a {\em knot point} if 
$x_{n}< x_{n+1}$ and $(x_n,y_n)$ does not lie on the straight line that joins its two immediate neighbors. By contrast, it is called a {\em jump point} if $x_{n}= x_{n+1}$ and $y_{n}\ne y_{n+1}$.
Finally, the ordered set $\mathbb{P}$ is called {\em minimal} if it contains only knot points and jump points, except for its two {\em boundary points} $(x_1,y_1)$ and $(x_N,y_N)$. Consequently, if $\mathbb{P}$ is minimal and has no jumps, then $f_{\mathrm{int}, \mathbb{P}}$ is a piecewise-linear spline that coincides with the canonical spline interpolant from Definition \ref{Def:Canonical}.

The minimal ordered set $\mathbb{P}=\{(x_n,y_n)\}_{n=1}^N$ is said to be nondecreasing if 
$y_1\le y_2 \le y_3\le  \dots  \le y_N$. In such a case, the piecewise-linear function $f_{\mathrm{int}, \mathbb{P}}$
is nondecreasing and invertible (in a set theoretical sense) with
$f^{-1}_{\mathrm{int}, \mathbb{P}}=f_{\mathrm{int}, \mathbb{P}^{-1}}$, where 
$\mathbb{P}^{-1}=\{(y_n,x_n)\}_{n=1}^N$. The latter property is the main ingredient that is used to establish our next result.

\begin{proposition}[Spline prox of a (weakly) convex potential]
\label{Proposition:weaklyconvex}
Consider the adaptive linear spline $f_{\rm spline}$ specified by \eqref{Eq:Linspline}.
If $f_{\rm spline}$ is nondecreasing,  
then there exists a unique continuous piecewise-quadratic potential $\phi: \R \to \R$ with $\phi(0)=0$ such that $f_{\rm spline}(x)={\rm prox}_\phi(x)$. 
Moreover, depending on the value of $s_{\max}$ given by \eqref{Eq:smaxspline}, the potential $\phi$ is endowed with the following properties.\\[-1ex]
\begin{enumerate}
\item If $s_{\max}\le 1$ (i.e., $f_{\rm spline}$ is firmly non-expansive), then $\phi$ is convex.
\item If $s_{\max} < 1$,  then $\phi$ is $\left(\frac{1}{s_{\max}}-1\right)$-strongly convex.
\item If $1 \le s_{\max} <\infty$, then $\phi$ is $\left(1-\frac{1}{s_{\max}}\right)$-weakly convex.
\end{enumerate}

\end{proposition}
\begin{proof} 
As explained above, we represent $f_{\rm spline}$ by the minimal ordered set $\mathbb{P}=\{(x_n,y_n)\}_{n=1}^N$ with $N=K+2$, $x_1=\tau_1-1$, $x_{n+1}=\tau_{n}$ for $n=1,\dots,K$, $x_{N}=\tau_K+1$, and $y_n=f_{\rm spline}(x_n)$.
Since $f_{\rm spline}=f_{\mathrm{int}, \mathbb{P}}$ is nondecreasing and piecewise-linear, it is invertible with $f^{-1}_{\mathrm{int}, \mathbb{P}}$ being piecewise-linear as well.
Let us now assume that $f_{\mathrm{int}, \mathbb{P}}(x)$ is the unique minimizer of $\tfrac{1}{2}(x-y)^2 + \phi(y)$, where $\phi$ is $\rho$-weakly convex.
The resolvent identity $f_{\mathrm{int}, \mathbb{P}}={\rm prox}_\phi=(\Identity + \partial \phi)^{-1}$ then yields
$ \partial \phi(y) =f^{-1}_{\mathrm{int}, \mathbb{P}}(y)-\{y\}$, where the subdifferential $\partial \phi$ (see Appendix A.3) can also be identified as the (single-valued) derivative $\phi'$. 
In other words, $\phi'$ is the piecewise-linear function
represented by the ordered set $\{(y_n,x_n-y_n)\}_{n=1}^N$, which is also minimal.
In particular, for $y \in [y_{n-1},y_{n}]$ (under the assumption that $y_{n}\ne y_{n-1}$), we have that
\begin{align*}
\phi'(y) &= (x_{n-1}-y_{n-1})+ 
 \frac{x_n-y_n - (x_{n-1}-y_{n-1}) }{y_n - y_{n-1} }   (y-y_{n-1} )\\
& =(x_{n-1}-y_{n-1}) + \left( \frac{1}{s_n} - 1 \right) (y-y_{n-1} )
\end{align*}
where $s_n= \frac{y_n-y_{n-1}}{x_n-x_{n-1}}>0$. This implies that $\phi''$ (the second derivative of $\phi$) is piecewise-constant with
\begin{align}
-1 < \frac{1}{s_{\max}} - 1 = \min\{\frac{1}{s_n} -1 \}_{n=2}^N   \le \phi''(y)
\end{align}
for almost all $y \in \R$ (except at the jumps where $y_{n-1}=y_{n}$), which confirms that $\phi$ is $\rho$-weakly convex. 
We can then apply the same method as in Proposition \ref{Proposition:Derivative} to identify the potential $\phi$, which is piecewise-quadratic and $\rho$-weakly convex with $\rho=\min(0,1-\frac{1}{s_{\max}}) < 1$. Moreover if $s_{\max}<1$, then $\phi$ is $\left(\frac{1}{s_{\max}}-1\right)$-strongly convex.
\end{proof}
We note that $\phi$ in Proposition \ref{Proposition:weaklyconvex} is a quadratic spline if and only if $f_{\mathrm{spline}}$ is 
\rev{strictly increasing}.
Otherwise, the corresponding $\phi'(y)$ will exhibit discontinuities at the critical points where $y_{n-1}=y_{n}$, which implies that $\phi$ is only differentiable once. 
An instructive example is the soft-threshold with parameter $\lambda$ which, in our formulation, is encoded as
$f_{\mathrm{int},\mathbb{P}_1}$ with 
$\mathbb{P}_1=\{(-\lambda -1,-1),(-\lambda,0),(\lambda,0),(\lambda + 1,1)\}$. The latter is a nondecreasing linear spline, albeit not a (strictly) increasing one. The corresponding derivative of the potential is $\phi_1'=
f_{\mathrm{int},\mathbb{P}'_1}$, where $\mathbb{P}'_1=\{(-1,-\lambda),(0,-\lambda),(0,\lambda),(1, \lambda)\}$ so that $\phi'(y)=\lambda\;\mathrm{sign}(y)$ exhibits a discontinuity at the origin. 
By  integration, we recover the well-known result $\phi_1(y)=\lambda |y|$ which, in the present setting, can be viewed as a borderline case of a continuous, piecewise-quadratic potential.

As in the previous example, it is often useful to reweight the strength of a regularizing potential via the use of a multiplicative factor $\lambda>0$.
In a Bayesian setting, $\lambda$ is typically set in proportion to the variance of the measurement noise (uncertainty). In our first scenario where the learned spline $f_\mathrm{spline}$ coincides with $\phi'$, the effect of such a reweighting is trivial
as $(\lambda \phi)'=\lambda \phi'=\lambda f_\mathrm{spline}$. In the second proximal setting, 
 the regularization effect of $\lambda$ is less trivial.

\begin{proposition}[Spline prox of a reweighted potential]
\label{Prop:reweightedProx}
Let $f_\mathrm{spline}=f_{\mathrm{ int}, \mathbb{P}}$ be a non-decreasing linear spline that is
described by the minimal ordered set of points $\mathbb{P}=\{(x_n,y_n)\}_{n=1}^N$.
Then, there exists a continuous piecewise-quadratic potential $\phi$ such that $f_\mathrm{spline}=\mathrm{prox}_\phi$ with
$\phi'=f_{\mathrm{int}, \mathbb{P}'}$,
where $\mathbb{P}'=\{(y_n,x_n-y_n)\}_{n=1}^N$. As for the reweighted potential $\lambda \phi$, it 
is: (i) convex for any $\lambda>0$ if $s_{\max}\le 1$; or (ii) $\rho$-weakly convex with $\rho=\left(\lambda-\frac{\lambda}{s_{\max}}\right)<1$
if $s_{\max}\ge 1$ and $0< \lambda \le \frac{ s_{\max}}{s_{\max}-1}$. In each of these cases,
\begin{align}\mathrm{prox}_{\lambda\phi}=f_{\mathrm{int}, \mathbb{P}_\lambda} \quad \mbox{ with }\quad
\mathbb{P}_\lambda=\big\{\big(\lambda x_n +(1-\lambda) y_n ,y_n\big)\big\}_{n=1}^N,
\end{align}
which is a nondecreasing linear spline as well.

\end{proposition}
\begin{proof}
As $f_\mathrm{spline}=\mathrm{prox}_\phi$ is a linear spline, $\phi'$ is piecewise-linear as well (but not necessarily continuous). To determine  
$\mathrm{prox}_{\lambda\phi}(x)=\arg \min_{y \inR} \tfrac{1}{2}(x-y)^2 + \lambda \phi(y)$ with $\lambda>0$, we use the optimality condition $0 \in \{y\} + \lambda \partial \phi(y)$ together with the explicit characterization
$\phi'=f_{\mathrm{int}, \mathbb{P}'}$ from the proof of Proposition \ref{Proposition:weaklyconvex}. This allows us to identify
$\mathrm{prox}_{\lambda\phi}$ as the inverse of the function $g(y)=y + \lambda \phi'(y)$, which is piecewise-linear and specified by the ordered set $\{\big(y_n,y_n+\lambda (x_n - y_n )\big)\}_{n=1}^N$. 
The condition for $g$ to be nondecreasing (and, hence, invertible) is
$y_{n-1}+\lambda (x_{n-1} - y_{n-1} )>y_{n}+\lambda (x_{n} - y_{n} )$ for $n=2,\dots,N$, which is equivalent to
\begin{align}
\label{Eq:sninequal}
s_n + \lambda (1 - s_n) >0
\end{align}
where $s_n= \frac{y_n-y_{n-1}}{x_n-x_{n-1}}\ge0$. Condition \eqref{Eq:sninequal}
 is satisfied for all $\lambda>0$ if $0\le s_n \le s_{\max} \le 1$. By contrast, if
$s_{\max}>1$, then we need to restrict $\lambda$ to the range $\left(0, \frac{s_{\max}}{s_{\max}-1}\right)$.
Within the range where $g$ is invertible, we find that
$\mathrm{prox}_{\lambda\phi}=g^{-1}$ is the piecewise-linear function specified by the ordered set  
$\mathbb{P}_\lambda=\big\{\big(\lambda x_n +{(1-\lambda) y_n} ,y_n\big)\big\}_{n=1}^N$, which also happens to be a spline since the new sampling locations $\lambda x_n +(1-\lambda) y_n$ are ordered and distinct (as direct consequence of the monotonity hypotheses $x_n< x_{n+1}$ and  $y_n\le y_{n+1}$).

As for the convexity properties of $\lambda \phi$, these are ruled by the monotonicity properties
of $\phi'=f_{\rm int, \mathbb{P'}}$ whose slopes are $s_n'=\left(\frac{1}{s_n}-1\right)$ for $n=2,\dots,N$ (see proof of Proposition \ref{Proposition:weaklyconvex}).
Consequently, $\inf_{y \inR }\lambda \phi'(y)=\min\{\lambda s'_n\}_{n=2}^N=\left(\frac{\lambda}{s_{\max}}-\lambda\right)$,
while $\sup_{y \inR }\lambda \phi'(y)=\left(\frac{\lambda}{s_{\min}}-\lambda\right)$.
This allows us to infer that the underlying potential $\lambda \phi$ will be convex for any $\lambda>0$ if and only if
$s_{\max}\le 1$. Likewise, if $s_{\max}>1$ and $0< \lambda \le \frac{ s_{\max}}{s_{\max}-1}$, then
$\lambda \phi$ with be $\rho$-weakly convex with $\rho=\left(\lambda-\frac{\lambda}{s_{\max}}\right)<1$.
\end{proof}

The interest of Proposition \ref{Prop:reweightedProx} is that it can help us adapt a given spline proximal operator to other experimental conditions. For instance, if the spline nonlinearity was trained as part of a Gaussian denoiser with
some fixed noise level $\sigma_1^2$, we can readily tune the nonlinearity to some other noise level $\sigma_2^2$ by changing the underlying potential $\phi$ to $\lambda \phi$ with $\lambda=\sigma_2^2/\sigma_1^2$. To illustrate the concept,
consider the soft-thresholding operator $\mathrm{prox}_{|\cdot|}$ which can be specified by the 
ordered set $\mathbb{P}=\{(-2,-1),(-1,0),(1,0),(2,1)\}$ with $s_{\max}=1$. Then, the application of Proposition \ref{Prop:reweightedProx} to $f=\mathrm{prox}_{|\cdot|}=f_{\mathrm{int},\{(-2,-1),(-1,0),(1,0),(2,1)\} }$
yields
$$\mathrm{prox}_{\lambda |\cdot|}=f_{\mathrm{int},\{(-\lambda-1,-1),(-\lambda,0),(\lambda,0),(\lambda+1,1)\} },$$
which is precisely the soft-threshold with parameter $\lambda$.

\section{Algorithmic framework for the Learning of Freeform Activations}
\label{Sec:Framework}

We now address the seemingly more challenging problem of learning freeform activation in deep neural networks. We first show that the theoretical result on the optimality of linear splines for fitting data subject to slope constraints (Theorem \ref{Theo:SplineFitConstrained}) can be applied to any layered architecture. 
We then present a practical way to discretize the underlying optimization problem, which can be effectively handled with the help of nonuniform B-splines.

\subsection{Learning Activations in Deep Neural Architectures}
\label{Sec:DeepExtension}
\rev{We consider a 
deep neural network of depth $L$ that is formally described as
\begin{align}
\label{Eq:DNN}
\V f_{\V \theta}=\V \sigma_L \circ \Op W_{L}\circ \V \sigma_{L-1} \circ \cdots\circ  \Op W_2 \circ \V \sigma_1\circ  \Op  W_1: \R^{N_{0}} \to \R^{N_{L}}.
\end{align}
This network results from the composition of linear transformations $\Op W_\ell: 
\V x \mapsto \M W_\ell \V x$ with $\M W_\ell \in \R^{N_{\ell}\times N_{\ell-1}}$ 
and of pointwise nonlinearities $ \V \sigma_\ell=(\sigma_{1,\ell},\dots, \sigma_{N_{\ell},\ell}): \R^{N_\ell} \to \R^{N_\ell}$, where
 $\sigma_{n,\ell}: \R \to \R$ denotes the activation of the $n$th neuron in the $\ell$th layer of the network. 
In the standard scenario, the shape of the neurons is fixed with $\sigma_{n,\ell}(x)=\sigma(x - b_{n,\ell})$, where $\sigma: \R \to \R$ is a shared profile (e.g., ReLU) and $b_{n,\ell}\in \R$ is an adjustable bias. Accordingly, the parameters of the network, collectively denoted by $\V \theta$,
consist of the linear weights $\M W_\ell$ and the biases $\M b_\ell=(b_{n,\ell})\in \R^{N_\ell}$ for $\ell=1,\dots,L$.
}

\rev{The classic training of \eqref{Eq:DNN} 
amounts to the tuning of
$\V \theta$  such that
$\V f_{\V \theta}(\M x_m) \approx\M y_m$ (without overfitting) for a representative set of data points $(\M x_m, \M y_m)\in \R^{N_{0}} \times \R^{N_{L}}$, $m=1, \dots, M$.
In practice, this is formulated as the minimization problem
\begin{align}
\V \theta^\ast \in \arg \min_{\V \theta \in \Omega, } \left(\sum_{m=1}^M E\big(\M y_m,  \M f_{\V \theta}(\M x_m)\big) + R(\V \theta) \right),
\label{Eq:deepnet}
\end{align}
where $E: \R^{N_{L}} \times \R^{N_{L}} \to \R$ is a convex loss function that quantifies the goodness of fit, $\Omega$ is the domain of acceptable weights, and
$R(\V \theta)$ is a regularization functional (such as weight decay) that makes the problem well-posed. The problem in  \eqref{Eq:deepnet} is then solved iteratively using stochastic gradient descent.
}

\rev{Our proposal is to augment the capabilities of \eqref{Eq:DNN} 
by allowing for freeform activations in the sense that we make the response curve of certain neurons 
learnable. We achieve this within the regularization framework of Theorem \ref{Theo:SplineFitConstrained}
by adding a term of the form $\lambda \sum_{(n,\ell) \in \mathbb{F}}\|\sigma''_{n,\ell}\|_{\Spc M}$ to the training loss in \eqref{Eq:deepnet}, and then by
jointly optimizing over $\V \theta \in \Omega$ and $(\sigma_{n,\ell})_{(n,\ell) \in \mathbb{F}}$ subject to the stability constraint $s_{\min}\le \sigma'_{n,\ell}(x)\le s_{\max}$.}

\rev{To prove that the optimal configuration is achieved with spline activations, we start with a single trainable neuron ($\mathbb{F}=\{(n,\ell)\}$) and denote by $\M f_{\V \theta,\sigma_{n,\ell}}: \R^{N_{0}} \to \R^{N_{L}}$ the multimensional mapping implemented by a deep neural network with weights $\V \theta$ 
and the generic activation of its $(n,\ell)$th neuron. 
This leads to the reformulation of the training problem as}
\begin{align}
(\V \theta^\ast,\sigma^\ast_{n,\ell}) \in \arg \min_{\V \theta \in \Omega, \sigma_{n,\ell} \in {\rm BV}^{(2)}(\R)} \sum_{m=1}^M E\big(\M y_m,  \M f_{\V \theta,\sigma_{n,\ell}}(\M x_m)\big) + R(\V \theta) +\lambda\|\sigma''_{n,\ell}\|_{\Spc M} \nonumber\\
\hspace*{3cm}\mbox{ s.t. }  s_{\min} \le \sigma_{n,\ell}'(x) \le s_{\max} \mbox{ a.e.}\label{Eq:deepnet2}
\end{align}
Under the assumption that \eqref{Eq:deepnet2} admits a  (not necessarily unique) minimizer, there exists a network configuration denoted by $\M f_{\V \theta^\ast,\sigma^\ast}$ that achieves the minimal cost with an optimal data term $D_{\rm opt}=\sum_{m=1}^M E(\M y_m,  \M z^\ast_m)$, where $\M z^\ast_m=\M f_{\V \theta^\ast,\sigma^\ast}(\M x_m), m=1, \dots,M$. Likewise, for each datum $\M x_m$, the optimal configuration imposes  at the $(n,\ell)$th neuron a specific
pair of input-output values $\big(x_{m,(n,\ell)},\ z_{m,(n,\ell)}\big)$ with $z_{m,(n,\ell)}=\sigma^\ast(x_{m,(n,\ell)})$. Now, the key insight is that one can replace the optimized activation $\sigma^\ast$ by any function $\sigma \in {\rm BV}^{(2)}(\R)$ such that
$z_{m,(n,\ell)}=\sigma(x_{m,(n,\ell)}), m=1,\dots,M$ (interpolation condition) without changing the primary part $D_{\rm opt} + R(\V \theta)$ of the total cost. Since the solution $\sigma^\ast$ must also meet the slope constraints, we can invoke the second part of Theorem \ref{Theo:SplineFitConstrained} to deduce the existence of a linear spline that achieves the global optimum. The argument
generalizes to multiple neurons, including configurations where the activation is shared by several neurons.
Consequently, we are able to extend our representer theorem for deep spline networks \cite{Unser2019c} to any scenario where one wishes to impose slope constraints (such as 1-Lip or invertibility) on the activations.

\subsection{Spline Parameterization and Training}
While the theory asserts that the optimal nonlinearities can all be encoded as nonuniform splines, we still need a practical way to determine the solution. In principle, one could plug the generic form of a linear spline given by
\eqref{Eq:Linspline} with $K$ sufficiently large into \eqref{Eq:OptProblem} or \eqref{Eq:deepnet}, and then minimize the cost functional by adjusting the weights and knot locations. Unfortunately, even in the simple scenario of data fitting, such a parametric optimization is difficult because  the dependency on the knot locations makes the problem highly non-convex. The other delicate point is the poor conditioning of the ReLU basis in \eqref{Eq:Linspline}: a small perturbation of $a_k$ tends to have a huge nonlocal effect on the overall shape of $f_{\rm spline}$.

To circumvent the first problem, we place an overabundance of knots at frozen locations $\tau_k$ on the real line and then rely on the sparsity-promoting properties of our regularizer to remove the unproductive ones. The crucial ingredient here is \eqref{Eq:TV2discrete},
which allows us to recast the problem as an $\ell_1$-norm minimization.
In the interest of efficiency and to avoid the conditioning issues associated with ReLU, we actually use an alternative representation that parameterizes the spline in terms of its sample values $f_n=f_{\rm spline}(t_n)$ at $N={K+2}$ ordered locations $t_n$ with $t_1 < \tau_1$, $t_{n+1}=\tau_{n}$ for $n=1,\dots,K$ \big(the spline knots in \eqref{Eq:Linspline}\big) and $t_{N}>\tau_K$. The corresponding parametric model (nodal representation) is 
\begin{align}
\label{eq:Bsplineexp}
f_{\rm spline}(x)= \sum_{n=1}^{N} f_n \varphi_n(x),
\end{align}
where the underlying (interpolating) basis functions are given by
\begin{align*}
\varphi_1(x)&=
\begin{cases}
\frac{t_2-x}{t_2-t_1}, & x \in I_1=(-\infty, t_2)\\
0,& \mbox{ otherwise}.
\end{cases}\\
\varphi_2(x)&=\begin{cases}
\frac{x-t_1}{t_2-t_1}, & x\in I_1\\
\frac{t_3-x}{t_3-t_2}, & x \in I_2=[t_2,t_3)\\
0,& \mbox{ otherwise}.\\
\end{cases}\\
\varphi_n(x)&=\begin{cases}
\frac{x-t_{n-1}}{t_{n}-t_{n-1}}, & x\in I_{n-1}=[t_{n-1},t_{n})\\
\frac{t_{n+1}-x}{t_{n+1}-t_{n}}, &x \in I_{n}=[t_{n},t_{n+1})\\
0, & \mbox{ otherwise},\\
\end{cases}\quad \mbox{ for } n=3,\dots,(N-2).\\
\varphi_{N-1}(x)&=\begin{cases}
\frac{x-t_{N-2}}{t_{N-1}-t_{N-2}}, & x \in I_{N-2}=[t_{N-2},t_{N-1})\\
\frac{t_{N}-x}{t_{N}-t_{N-1}},& x\in I_{N-1}=[t_{N-1},+\infty)\\
0,& \mbox{ otherwise}.
\end{cases}\\
\varphi_N(x)&
=\begin{cases}
\tfrac{x-t_{N-1}}{t_{N}-t_{N-1}}, & x \in I_{N-1}\\
0,& \mbox{ otherwise}.
\end{cases}
\end{align*}
There, the real line is partitioned as
$\R=\bigcup_{n=1}^{N-1} I_n$
with the $\varphi_n$ for $2<n<(N-2)$ being nonuniform triangular B-splines supported in $[t_{n-1}, t_{n+1})=I_{n-1} \cup I_n$  (see Fig. \ref{Fig: InterpolBsplines}).
The triangular splines are complemented with four one-sided basis functions that extend linearly towards $\mp\infty$ to enable the proper extrapolation of the boundary values. Even though these four boundary functions are not compactly supported, the remarkable feature of our representation is that the evaluation of \eqref{eq:Bsplineexp} for any given $x \in \R$ involves at most two active basis functions. This makes the computation very efficient and independent of $N$. 
\begin{figure}
 \includegraphics[width = 10cm]{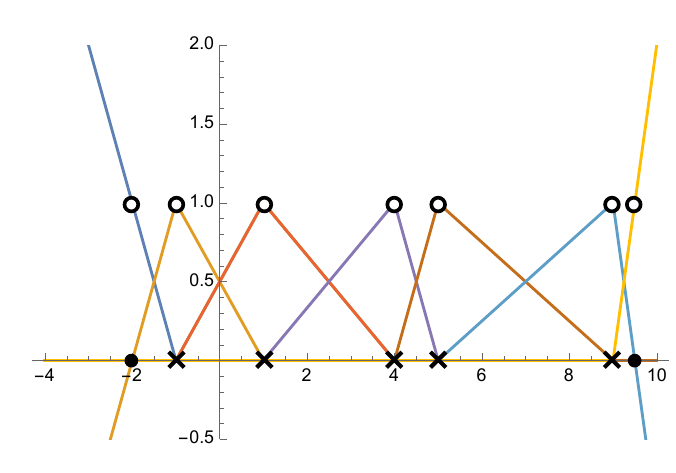}
 \caption{\label{Fig: InterpolBsplines}Interpolating basis functions associated with the grid points $\V t=(-2,-1,1,4,5,9,9.5)$. The locations of the spline knots are marked by crosses.}
\end{figure}

Given the nodal values $\M f=(f_n) \in \R^N$ of the spline, we calculate its slopes $s_n=\frac{f_n-f_{n-1}}{t_{n}-t_{n-1}}$ for $n=2,\dots,M$, and store them in the vector $\M s=(s_n) \in \R^N$ with a repeated value $s_1=s_2$ for $n=1$. This is formalized as $\M s=\M D_\V t \M f$, where $\M D_\V t\in \R^{N \times N}$ is the divided-differences matrix associated with the spline grid $\V t=(t_1, \dots,t_N)$. The vector $\M s$ informs us on the slope excursion of the spline (minimum and maximum)  and also yields the regularization cost $\mathrm{TV}^{(2)}(f_{\rm spline})=
\sum_{n=2}^N |s_n-s_{n-1}|$
(see \eqref{Eq:TV2spline}), which may be written as $\mathrm{TV}^{(2)}(f_{\rm spline})=\|\M D \M D_\V t\M f\|_{\ell_1}$, where $\M D$ is the finite-difference matrix of size $N$.

Conversely, we may convert back the slopes $\M s$ to the nodal values $\M f$ up to a global summation constant, which may be identified as the value of $f_1$. Specifically, we have that
\begin{align}
\label{Eq:Inversediff}
f_n=f_1+\sum_{n=2}^N   s_n(t_n-t_{n-1})=f_{n-1} + s_n(t_n-t_{n-1}),
\end{align} 
where the rightmost relation suggests a very efficient recursive computation of complexity $O(1)$.

In our implementation, we impose the slope constraints by applying a projector ${\rm Proj}_{\rm slope}$ that clips the values of the slope of $f_{\rm spline}$ to the range $[s_{\min},s_{\max}]$, while preserving the mean of the nodal values $f_n=f_{\rm spline}(t_{n})$. The action of this clipping operator on the spline coefficients  $\M f$ is described as
\begin{align}
\label{Eq:ProjSlope}
{\rm Proj}_{\rm slope}: \M f \mapsto \M D_\V t^\dagger {\rm clip}_{[s_{\min},s_{\max}]}(\M D_\V t \M f)+ \M 1 \frac{1}{N} \sum_{n=1}^N f_n,
\end{align}
where $\M D_\V t^\dagger$ is the unique right inverse of $\M D_\V t$ that imposes the boundary condition 
$\M 1^\Top\M D_\V t^\dagger \M s = 0$ for all $\M s \in \R^N$.
We note that $\M s \mapsto \M D_\V t^\dagger\M s$ has a fast implementation that is given by
the right-hand side of \eqref{Eq:Inversediff} modulo a proper adjustment of $f_1$.

With the proposed choice of discretization and $E(f,y)=|f-y|^2$ to keep the notation simple, we recast \eqref{Eq:OptProblem} as a finite-dimensional optimization
 \begin{align}
\label{Eq:OptProblemDis}
\M f^\ast \in \arg \min_{\M f\in \R^N} \left(\sum_{m=1}^M \| \M y - \M S \M f\|_2^2 + \lambda \|\M D \M D_\V t \M f\|_{\ell_1}\right) \\
\hspace*{3cm}\mbox{ s.t. }  s_{\min} \le [\M D_\V t\M f]_n \le s_{\max},\quad n=1,\dots,N,\nonumber\end{align}
where the underlying sampling matrix $\M S \in \R^{M \times N}$ is specified by $[\M S]_{m,n}=\varphi_n(x_m)$.
Equation \eqref{Eq:OptProblemDis} is reminiscent of the LASSO problem \cite{Tibshirani2013} encountered in compressed sensing and  is amenable to an efficient implementation using the standard tools of convex optimization \cite{Parikh_proximal_2014}. 

To handle more involved joint optimization problems of the type described by \eqref{Eq:deepnet}, we have developed a corresponding module for adaptive-spline neurons with second-order TV regularization that can be inserted in any neural-network architecture and trained efficiently using the customary optimization tools of machine learning (back-propagation \& Adam).
The present scheme extends the deep-spline framework \cite{Bohra2020b} by incorporating new features to accommodate nonuniform knots and enforce slope constraints. We achieve the latter almost seamlessly by relying on the right-hand side of \eqref{Eq:ProjSlope} (which is auto-differentiable with respect to $\M f$) to explicitly parameterize the linear splines that fulfill the constraint.
%
%
%
\rev{\subsection{Function-Fitting Experiments}
To probe the benefit of our method, we compared several parameterization of CPWL functions for the basic task of function fitting in 1D.  We considered four alternative spline models with the same number $K$ of knots (or neurons in the case of the RELU networks):
\begin{enumerate}
    \item B-spline parameterization, as described by \eqref{eq:Bsplineexp};
    \item gridded ReLUs with skip connection, as described by \eqref{Eq:Linspline};
    \item MLP-FB: two-layer neural network with fixed biases;
    \item MLP: two-layer neural network with learnable biases.
\end{enumerate}
Models 1-3 have the same knots (equally spaced on $[-3,3]$ with $K=100$), which ensures that they all span the same space of linear splines. 
Model 4 has more expressivity, as it can also learn the position of the knots/biases. The models are constrained as in \eqref{Eq:OptProblem} via the inclusion of a ${\rm TV}^{(2)}$ penalty
with strength $\lambda$. This can also be achieved for Models 3 and 4 because of the remarkable equivalence between ${\rm TV}^{(2)}$ regularization
and weight decay \cite{Parhi2020}.
}
\begin{table}[]
\begin{tabular}{l|c|c|c}
\hline \hline
\hspace*{0em}Model & $\lambda=0$  & $\lambda=10^{-6}$ & $\lambda=10^{-4}$ \\ \hline
B-splines & ${\bf 2.18 \cdot 10^{-5}}$              & ${\bf 1.39 \cdot 10^{-4}}$
          & ${\bf 9.79\cdot10^{-3}}$  \\ 
Gridded ReLUs & $ 1.00 \cdot 10^{-4}$              & $1.40 \cdot 10^{-4}$              & $9.95\cdot10^{-3}$              \\ 
2-layer NN fixed bias & $6.60 \cdot 10^{-4}$              & $1.02 \cdot 10^{-3}$              &  $6.88\cdot10^{-2}$          \\ 
2-layer NN learned bias & $2.47\cdot10^{-4}$ & $3.63\cdot10^{-4}$ &  $3.81\cdot10^{-2}$\\ \hline
\hline
\end{tabular}
 \caption{
 \label{tab:fitting_exp}
\rev{Loss (data term + regularization) achieved by the four models after training with the regularization strength $\lambda$. The data-fitting term is evaluated by sampling the trained model at 10000 evenly spaced locations.}}
\end{table}

\rev{
We trained these models in Pytorch to fit the function $f(x) = \cos(10x)\exp(-x^2)$ in the range $[-3, 3]$.
For a fair comparison, we used the same optimization parameters in all scenarios: 1 million steps and a batch size of 1000.  The resulting losses 
are shown in Table ~\ref{tab:fitting_exp}. We observe that the B-splines and the gridded ReLUs have very close performance, while the MLPs are always doing worse. 
The numbers shown in bold can be taken as ground truth because the underlying fitting problem is convex and the iterative optimization has converged.
However, what strikingly distinguishes these models is the rate of decay of the testing loss, as
visualized in Figure~\ref{Fig:mse_loss}.
\begin{figure}
\center
 \includegraphics[width = 12cm]{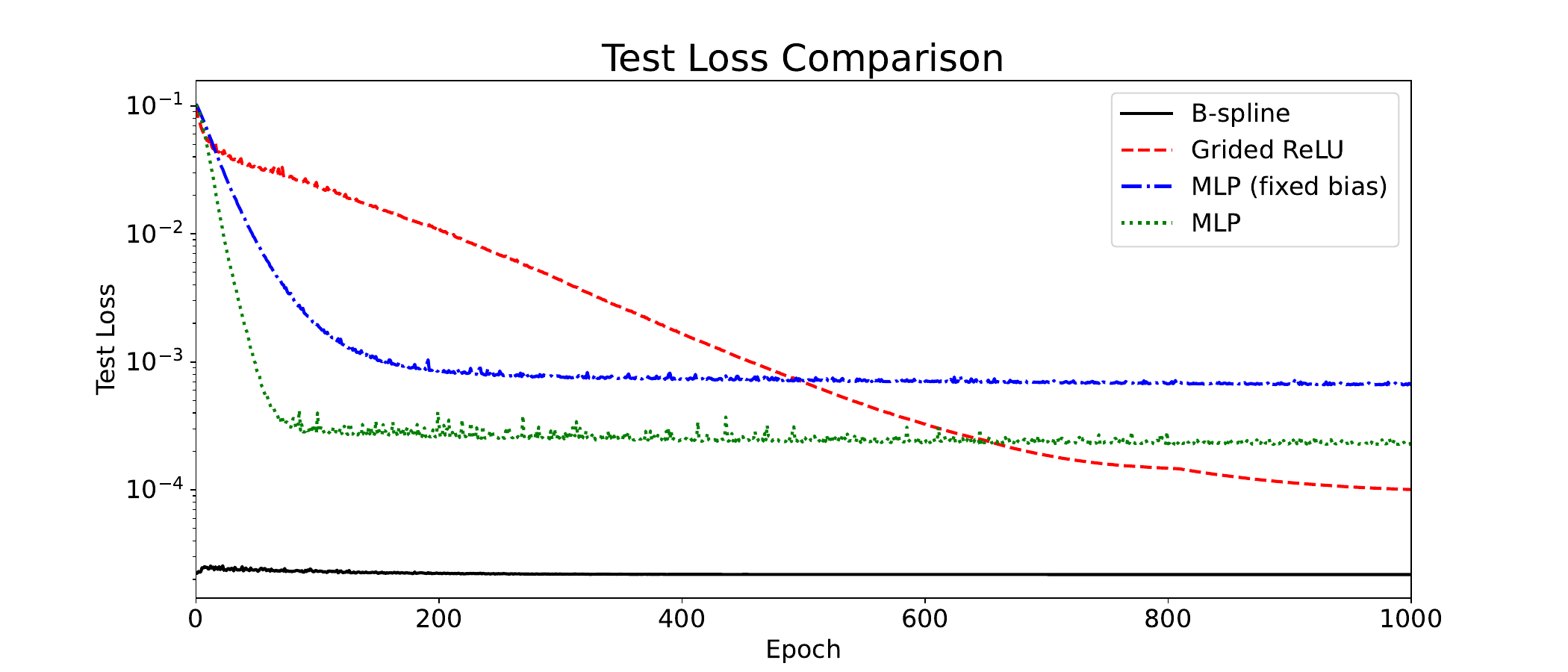}
 \caption{\label{Fig:mse_loss} Evolution of the loss during the training procedure, with an epoch corresponding to 1000 steps of SGD.}
\end{figure}
It is remarkable that the B-spline model (the bottom curve that is virtually flat) converges almost instantaneously. 
The evolution of the test loss  also suggests that the gridded ReLU model eventually converges, but at a rate that is many orders of magnitude slower than that of the B-splines.
While the two MLP models are initially able to decrease the error faster than gridded ReLU, 
they eventually stall and are unable to reach the minimum. This shows how much the local nature of the B-splines makes the training easier, not to mention that each iteration is much faster
since each data point affects two basis functions only, as opposed to the (almost) full set of ReLUs for Models 2-4. Finally, unlike Models 2-4, the B-spline representation
lends itself particularly well to the incorporation of
the kind of slope constraints supported by the present theory.
}

\section{Learned Potentials for Image Reconstruction}
\label{Sec:Denoising}
Next, we consider the application of our framework to 
the resolution of inverse problems in imaging. Given the noisy data $\M y \in \R^M$ and the linear measurement model $\M y=\M H \M x + {\rm ``noise"}$ with a known system matrix $\M H \in \R^{M \times N}$, the task is to recover the signal 
$\M x \in \R^N$. 

\subsection{Learned Gradients} 
\label{Sec:Analysis}
Our first approach is a variational formulation inspired by the ``fields of experts'' model
with a learned regularization functional \cite{Roth2009,Chen2017}. 
To that end, we specify our desired signal reconstruction as the solution of the regularized least-squares problem in
\begin{align}
\label{Eq:Inverse1}
\M x^\ast \in \arg \min_{\M x \in \R^N} \left(\frac{1}{2}\|\M y- \M H \M x\|_2^2 + 
\underbrace{ \sum_{i=1}^{I} \langle \M 1, \V \phi_i(\M W_i \M x)\rangle}_{R(\M x)} \right)
\end{align}
with a pooled regularization where each sub-term has
its own filter and its own univariate potential $\phi_i: \R \to \R^+$ (e.g., $\phi_i(z)=\lambda_i |z|$).
Specifically, the filter in regularization channel $i$ is represented 
by the convolution matrix $\M W_i\in \R^{N \times N}$, while $\V \phi_i(\M z)=\big(\phi_i(z_1),\dots, \phi_i(z_N)\big)$ is a vector-valued potential that yields a per-channel contribution $\langle \M 1, \V \phi_i(\M z)\rangle=\sum_{n=1}^N\phi_i(z_n)$.
The complete filterbank $\M W=[\M W_1 \cdots \M W_I] \in \R^{N \times(N \times I)}$ is also spectrally normalized 
 to have a direct control of weak convexity via the bounds on $s_{\min}$.

%
Under the assumption that the $\phi_i$ are convex and differentiable with $\psi_i=\phi_i'$, we can solve \eqref{Eq:Inverse1}
iteratively by steepest descent. This yields the iterative reconstruction algorithm
\begin{align}
\label{Eq:Inverse1Algo}
\M x^{(n+1)} = \M x^{(n)} - \gamma \left(\sum_{i=1}^{I} \M W_i^\Top {\V \psi_i}(\M W_i \M x^{(n)})+\M H^\Top\big(\M H \M x^{(n)}- \M y\big)\right)  \quad \mbox{with} \quad \V \psi_i=\V \phi'_i,
\end{align}
which can be interpreted as a recurrent neural network. 

To learn the regularization $R: \R^N \to \R$ in \eqref{Eq:Inverse1}
that  best represents a given class of signals/images, we follow the strategy of \cite{Goujon2023} and consider a basic denoising task with $\M H=\M I$ where the signal is corrupted by additive white Gaussian noise. To adjust the underlying model such as to achieve the best denoising on a representative set of images, we unroll the neural network \eqref{Eq:Inverse1Algo} or use deep equilibrium \cite{Gilton2021} to learn the
 filters $\M W_i$ and the nonlinearities $\psi_i$, which are shared across iteration layers. \rev{The only modification to the procedure described in \cite{Goujon2023} 
 is the incorporation of the (weak) convexity constraint.} We achieve this with the help of the projector \eqref{Eq:ProjSlope}, in adequation with  Proposition \ref{Proposition:Derivative}. 
 Once the optimal filters $\M W_i$ and spline activations $\psi_i=\phi_i'$ are known, we use \eqref{Eq:spline2}
 to deduce
 the regularization cost $R(\M x)= \sum_{i=1}^{I} \langle \M 1, \V \phi_i(\M W_i \M x)\rangle$ that works best on the denoising task. Since this regularization functional captures the prior statistical distribution of the signal, it is 
 also applicable
to the resolution of more general inverse problems (under a maximum-a-posteriori interpretation of the reconstruction process). This is to say that we can use our pretrained filters and nonlinearities to solve the whole variety of linear inverse problems specified by \eqref{Eq:Inverse1} by running the generic steepest-descent algorithm described by \eqref{Eq:Inverse1Algo} with an appropriate step $\gamma$.

For illustration purposes, we run a series of denoising experiments on natural images with increasing levels of Gaussian noise. The experimental protocol was the same as in \cite{ZZCMZ2017,Ryu2019plug} with the training set consisting of 
238'400 patches of size $(40\times 40)$ extracted from 400 images of the BSD500 dataset \cite{Arbelaez2011}. We varied
the $\rho$-weak convexity constraint from $\rho=0$  (convexity, {\rev{as in \cite{Goujon2023}) to $\rho=1$, which is the limit of convexity for the optimization problem \eqref{Eq:Inverse1} with $\M H=\M I$.}
In our framework, this corresponds to $s_{\max}=\infty$, and to have $s_{\min}$ decrease from $0$ (monotonicity) to $-1$ (weak monotonicity).
Our recurrent neural network involves $I=60$ convolution channels with filters of size $(13 \times 13)$. To facilitate the variational interpretation,
we used nonlinearities of the form $\psi_i(z)=\frac{1}{\alpha_i}\psi(\alpha_i z)$ with a single shared profile $\psi: \R \to \R$  and a scaling parameter $\alpha_i$ that is trained on a per-channel basis. The resulting 
signal-to-noise ratio curves as a function of $\rho$ are shown in Figure \ref{Fig:Denoising1}. We observe that, by relaxing the convexity constraint, we can get a performance improvement of the order of +0.5dB in all cases, albeit with a tendency to saturation in the low noise regime. We note that these results are competitive with those of BM3D \cite{Dabov2007} and among the very best within the category of denoisers specified by a convex optimization problem.
\rev{For comparison, the denoising performance of BM3D for this dataset is (37.54 dB, $\sigma$=5/255), (31.11dB, $\sigma$=15/255), and (28.60 dB, $\sigma$=25/255).}
As for the learned nonlinearity (bottom panel in Figure \ref{Fig:Potentials}), they are all antisymmetric with a linear behaviour around the origin (even if this is hardly 
apparent on the graph because of the strong underlying slope) and an asymptotic tendency to clip or even suppress (in the weakly convex scenario) inputs whose magnitude is higher than some implicit threshold.
Interestingly, the learned potential in the convex case is close to a 
$\ell_1$-norm (i.e., $\phi_i(z) \propto |z|$), while the ones for large $\rho$ have a concave profile that can be expected to promote sparsity even further.
\rev{We also found the optimized denoisers to be robust and applicable to a wide variety of images without any need for retraining. 
Moreover, we did deploy our pretained weakly-convex regularizers for image reconstruction (CT and MRI), and were able to obtain competitive results within the class of reconstruction algorithms with theoretical guarantees (e.g. consistency and stability of the reconstruction) \cite{Goujon2024Weakly}.}

\begin{figure}
 \includegraphics[width = 13cm]{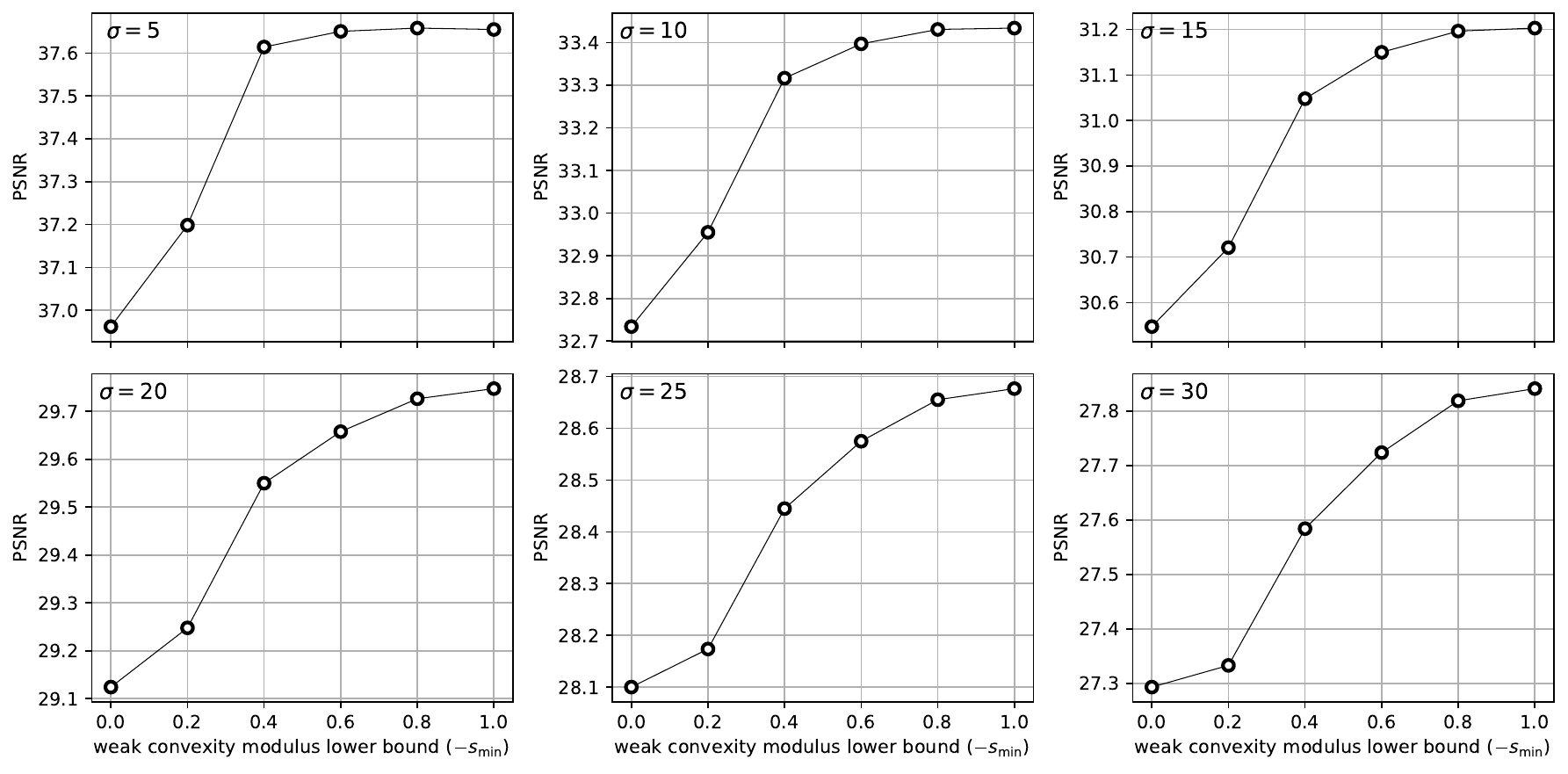}
 \caption{\label{Fig:Denoising1} Performance summary of variational denoising with trainable analysis filters as a function of $\rho$ (modulus of weak convexity).}
\end{figure}
\begin{figure}
 \includegraphics[width = 13cm]{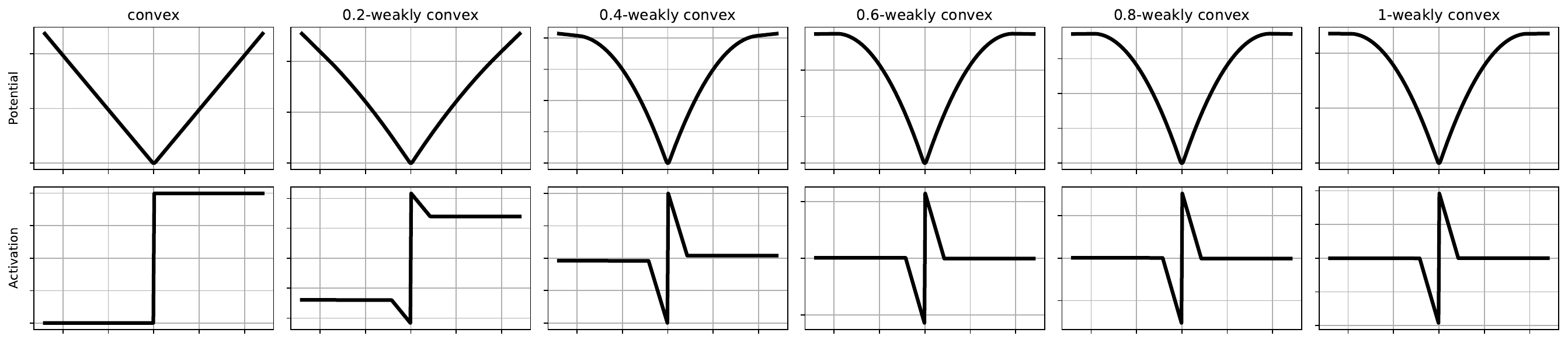}
 \caption{\label{Fig:Potentials} Learned potential $\phi$ and its derivative $\psi$.}
\end{figure}

\subsection{Learned Proximal Operators}
\begin{figure}
 \includegraphics[width = 13cm]{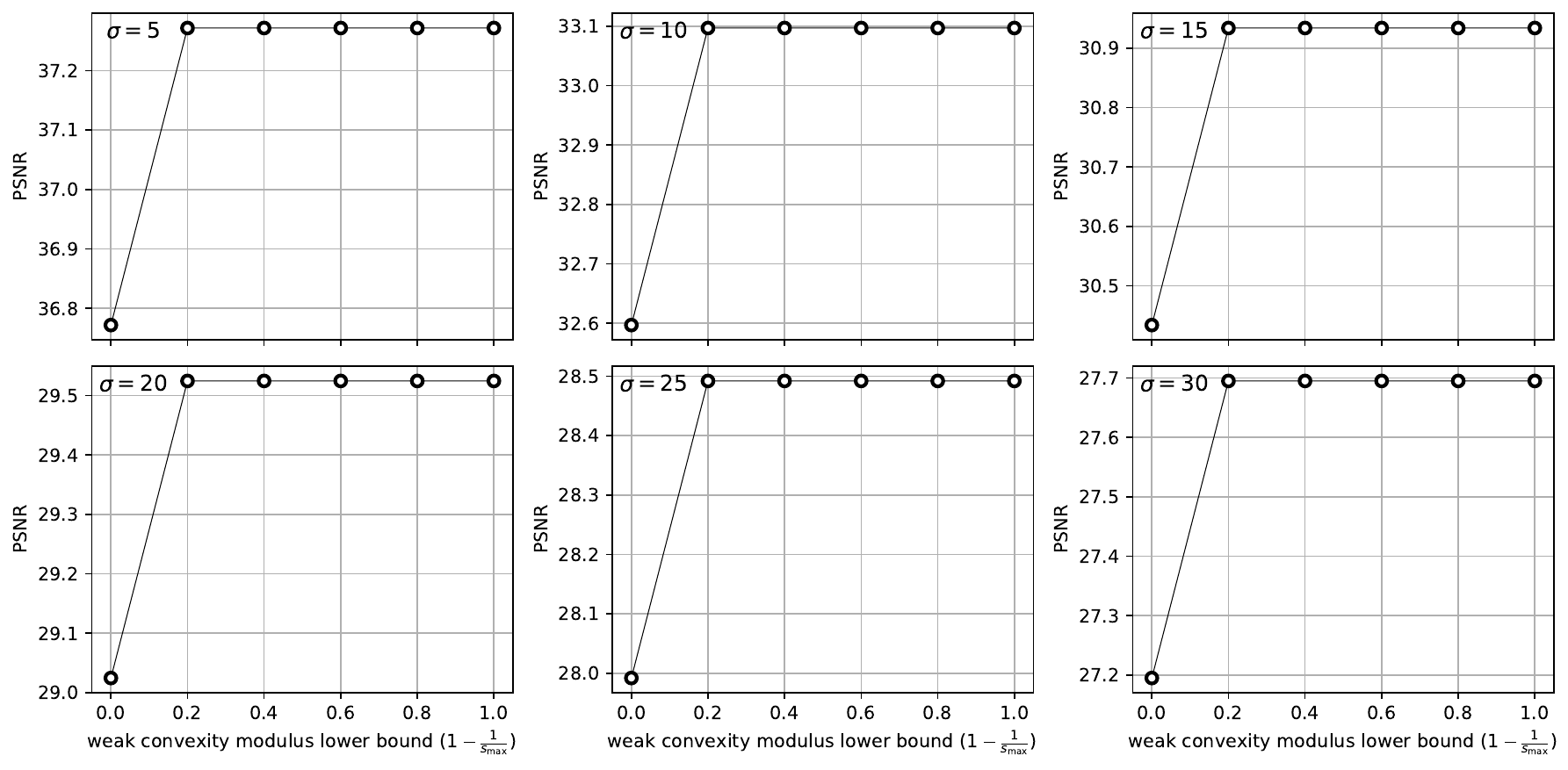}
 \caption{\label{Fig:Denoising2} Performance summary of variational denoising with trainable synthesis filters as a function of $\rho$ (modulus of weak convexity).}
\end{figure}

As alternative to the steepest-descent approach in Section \ref{Sec:Analysis}, we now demonstrate the usage of learned proximal operators.
To that end, we consider a synthesis formulation of the problem with a learnable filterbank $\M W=[\M W_1 \cdots \M W_I] \in \R^{N \times(N \times I)}$, where $\M W_i \in \R^{N \times N}$ ($i$th filter/block Toeplitz/circulant matrix) and a regularization functional that acts on the coefficients of the signal. 
We then reconstruct our signal as $\M x^\ast=\M W\M z^\ast$ where the optimal coding vector $\M z^\ast \in \R^{N \times I}$ is such that\begin{align}
\M z^\ast \in \arg \min_{\M z=(\M z_1,\dots,\M z_{I}) \in \R^{N\times I}} \left(\frac{1}{2}\|\M y- \M H \M W\M z\|_2^2 + 
\sum_{i=1}^{I} \langle \M 1, \V \phi_i(\M z_i)\rangle \right).
\label{Eq:SynthesisInverseProb}
\end{align}
There, the system matrix $\M H \in \R^{M \times N}$ is identical to that in \eqref{Eq:Inverse1}, while 
the regularization maps $\V \phi_i: \R^N \to \R^N$ retain the same structure, with a shared trainable potential $\phi_i: \R \to \R$ in each channel $i$ . 
The form of \eqref{Eq:SynthesisInverseProb} is standard in compressed sensing; it lends itself to an efficient resolution
using the popular-proximal gradient algorithm (a.k.a.\ backward-forward splitting).
The latter requires the gradient of the data term with respect to $\M z=(\M z_1, \dots, \M z_I)$, which is given by $\M W^\Top \M u$ with\begin{align}
\M u=\M H^\Top\big(\M H \M W\M z- \M y\big)\in \R^N.
\end{align}
The other important quantity is the  Lipschitz constant $L$ of this gradient,
which is bounded by the maximal singular value of $\M H$ under our working hypothesis that $\M W$ is spectrally normalized.
 This then yields the iterative reconstruction algorithm
 \begin{align}
 \label{Eq:IterProxGrad}
\M z_i^{(n+1)} = \V f_i\left(\M z_i^{(n)} - \frac{1}{L}\M W_i^\Top \M H^\Top\big(\M H \M W\M z^{(n)}- \M y\big)\right)
\end{align}
with $\V f_i=(f_i,\dots,f_i)
: \R^{N} \to \R^{N}$, where the shared nonlinearity
$f_i=\mathrm{prox}_{\tfrac{1}{L}\phi_i}: \R \to \R$ is the univariate proximal map associated with channel $i$. Again,  \eqref{Eq:IterProxGrad} for $i=1,\dots,I$ specifies a recurrent neural network with freeform activations $f_1, \dots,f_{I}$ that can be trained on a denoising task (with $\M H=\M I$ and $L=1$).
We rely on
Proposition \ref{Proposition:weaklyconvex} to ensure that the $f_i$ are admissible proximal operators. This gives the appropriate bound on $s_{\max}$ in addition to the monotonicity condition $s_{\min}\ge 0$. Here too, we can push the framework into the weakly convex regime by releasing the boundedness constraint on $s_{\max}$. 

We have applied the same protocol as in Section \ref{Sec:Analysis} to train the proximal network \eqref{Eq:IterProxGrad} for a basic denoising task. The outcome of this denoising experiment is summarized in Fig. \ref{Fig:Denoising2}. Once again, the transition into the weakly-convex regime is beneficial with an almost systematic gain of 0.5 dB, although there a strong tendency to saturation beyond $\rho=0.2$.
The results are promising, but not at the level of the ones reported in Section \ref{Sec:Analysis} where the regularization acts in an ``analysis'' mode. With the current filtering architecture, there seems to be a 0.3 to 0.1dB drop of performance (depending on the level of noise) when switching from an analysis to a synthesis configuration.
We attribute this behaviour to the greater difficulty in training the
synthesis filterbank with the stochastic-gradient procedure taking much longer to converge. This is consistent with the documented observation that convolutional sparse coding (CSC)---the special case of \eqref {Eq:SynthesisInverseProb} with a fixed nonlinearity (soft-threshold)---is not the best denoising technique among the dictionary-based methods  \rev{\cite{Chen2014Learning,Carrera2017,Plaut2019,Simon2019,Vedaldi2020}.
This suggests that there is still room for exploration in this area by considering trainable variants of other popular iterative schemes (e.g. primal-dual or ADMM) that
rely on scalar proximal maps \cite{Condat2012,Boyd2017,Monga2021}.}

\section{Conclusion}
We have presented a general framework for the controlled learning of pointwise nonlinearities in neural networks and, by extension, in any layered, trainable computational architecture. While our key result on the optimality of linear splines (Theorem \ref{Theo:SplineFitConstrained}) is stated and proved for a generic 1D data fitting problem subject to slope constraints, it has much further reaching consequences. Indeed, we have shown that  the joint optimization of the linear layers and activation functions of a deep neural network generally also yields adaptive linear spline solutions. 
We have then addressed the issue of the implementation of such trainable activations by developing a computational toolbox that relies on the use of nonuniform B-splines. A remarkable feature of the proposed parameterization is that each data point only activates two basis functions. This makes the training of the neural network (including the back-propagation step) very efficient. 
Our extended version of the deep-spline toolbox is available at \url{https://github.com/Biomedical-Imaging-Group/DeepSplines}.

Our projection-based mechanism to limit the slope excursion of the learned nonlinearities makes it very easy to impose certain desired properties.
For instance, by setting $(s_{\min},s_{\max})=(-1,1)$, we impose 1-Lipschitz stability,
which is the layer-wise condition that guarantees the convergence of plug-and-play schemes such as \cite{Ryu2019plug}. Likewise, for $(s_{\min},s_{\max})=(0,1)$, we constrain the nonlinearity to be firmly non-expansive and, hence, to be the proximal operator of a convex potential---the standard condition of usage for proximal-gradient algorithms \cite{Parikh_proximal_2014}. Another option is to set $(s_{\min},s_{\max})=(\epsilon,\infty)$ with $\epsilon>0$ arbitrarily small, which then yields a nonlinearity that is increasing and therefore invertible. Since the inverse of a linear spline is itself a linear spline, such nonlinear modules can be readily incorporated in the design of invertible flows \cite{
Kruse2021}.

\appendix
\section*{Appendix: Basic Notions from Convex Analysis}
\label{Sec:Appendix}

\subsection*{A.1 Classic Framework}
\begin{definition}[Classic convexity]
\label{Ref:Convex}
A function $f: \R^d \to \R$ is said to be
\begin{enumerate} {}
\item {\em convex} if, for all $\lambda\in ( 0,1)$ and all $\M x_1,\M x_2 \in \R^d$ such that $\M x_1\ne \M x_2$,
$$f\big(\lambda\M x_1+(1-\lambda)\M x_2\big)\le \lambda f(\M x_1)+(1-\lambda)f(\M x_2);$$
\item {\em strictly convex} if $$f\big(\lambda\M x_1+(1-\lambda)\M x_2\big)< \lambda f(\M x_1)+(1-\lambda)f(\M x_2);$$
\item {\em $\rho$-strongly convex} with $\rho>0$ if $\M x\mapsto -\rho\|\M x\|_2^2 +f\big(\M x)$ is convex;
\item {\em $\rho$-weakly convex} with $\rho>0$ if $\M x\mapsto \rho\|\M x\|_2^2 +f\big(\M x)$ is convex.
\end{enumerate}
\end{definition}
As suggested by the nomenclature, one has the following chain of implication: $\rho$-strong convexity $\Rightarrow$ strict-convexity $\Rightarrow$ convexity
\cite{Roberts1974convex}.
Also, a convex function $f: \R^d \to \R$ has the convenient property of being continuous (and a fortiori l.s.c.) over $\R^d$. If, in addition, $f$ is differentiable, then
its convexity (Item 1) is equivalent to the first-order condition \cite{Roberts1974convex,Boyd2004convex}
\begin{align}
\label{Eq:firstorder}
\forall x,y \inR^d:\quad f(y) \le f(x) + \V \nabla f(x)^\Top(y-x)
\end{align}
where $\V \nabla f: \R^d \to \R$ is the gradient of $f$.
\subsection*{A.2 Extended Framework}The notion of convexity admits a natural topological extension for functions $f: \R^d \to \overline{\R}$ whose domain
is the extended real line $\overline{\R}=\R \cup \{+\infty\}$. Such functions are often used to impose hard constraints such as the inclusion in some closed set $C \subset \R^d$. The typical example is the barrier function
\begin{align}
\label{Eq:Barrier}
i_C(\M x)=\begin{cases}
0,& \mbox{ if } \M x \in C\\
+\infty,& \mbox{otherwise.}\end{cases}
\end{align}
The relevant tool for the characterization of such functions (including the conventional ones) is the
epigraph, which is the subset of $\R^d \times \R$ defined by
 $${\rm epi}f=\{(\M x, w) \in \R^d \times \R \mbox{ s.t. } f(\M x) \le w \mbox{ for some } \M x \in \R^d\}.$$
The latter may be visualized as the area on or above the graph of the function. This alternative description then calls for the following extended definitions.


\begin{definition}
\label{Ref:ConvexExt}
A function $f: \R^d \to \overline{\R}$ is said to be
\begin{enumerate} {}
\item {\em convex} if ${\rm epi}f$ is a convex subset of $\R^d \times \R$;
\item {\em strictly convex} if ${\rm epi}f$ is a strictly convex subset of $\R^d \times \R$;
\item {\em closed} if ${\rm epi}f$ is a closed subset of $\R^d \times \R$;
\item {\em proper} if there exists at least one $\M x_0 \in \R^d$ such that $f(\M x_0)< +\infty$;
\item {\em coercive} if $f(\M x) \to +\infty$ as $\|\M x\|\to +\infty$;
\item {\em lower-semicontinuous} (l.s.c.) at a point $\M x_0$ if, for every $y< f(\M x_0)$, there exists an 
$\epsilon>0$ such that $y < f(\M x)$ for every $\M x \in B_\epsilon(\M x_0)=\{\M x \in \R^d: \|\M x-\M x_0\|_2 < \epsilon\}$.
\end{enumerate}
\end{definition}
Since $\R \subset \overline{\R}$, these definitions are also applicable to ``ordinary'' functions $f: \R^d \to \R$, in which case the characterizations in Item 1-2 of Definitions \ref{Ref:Convex} and \ref{Ref:ConvexExt} are equivalent. We also note that the property of $f$ being l.s.c.\ on $\R^d$ is equivalent to
$f$ being a closed function on $\R^d$.
In particular, the barrier function $i_C$ specified by \eqref{Eq:Barrier} is l.s.c.\ (or closed) if and only if
$C$ is a closed subset of $\R^d$. Likewise, $i_C: \R^d \to \overline{\R}$ is convex if and only if $C$ is a convex subset of $\R^d$. Finally,
$i_C$ is coercive if $C$ is a bounded subset of $\R^d$. 

The key properties for optimization theory are the coercivity and the l.s.c./ {closedness} of $f$; together, they imply the existence in $\R$ of the minimum ${\inf_{\M x\in \R}f(\M x)}>-\infty$. The convexity property is remarkable in that it ensures that any local mimum of $f$ is also a global minimum. Finally, the combination of l.s.c.\ and strict convexity ensures that the minimum is unique.

\subsection*{A3. Set-Valued Operators and Subdifferential}
The power of a set $\Spc X$ (here, the vector space $\Spc X=\R^d$) is the set of all subsets of $\Spc X$ denoted by $2^\Spc X$.
A set-valued operator $\Op T: \Spc X \to 2^\Spc X$ maps each element of $\Spc X$ into a set of $\Spc X$.
If $\Op T(x)$ is a singleton for all $x \in \Spc X$, then $\Op T$ is single-valued over $\Spc X$ and it can be identified as a conventional function $\Op T: \Spc X \to \Spc X$ (with a slight abuse of notation).
The graph of an operator $\Op T: \Spc X\to 2^{\Spc X}$ is defined as
\begin{align}
{\rm graph}\Op T= \{(x,y) \;\big|\; x \in \Spc X, y \in \Op T(x)\}.
\end{align}
This notion provides us with a convenient characterization of the inverse $\Op T^{-1}: 2^\Spc X \to 2^\Spc X$ of a set-valued operator:
\begin{align}
{\rm graph}(\Op T^{-1})= \{(y,x) \;\big|\; (x,y) \in {\rm graph}\Op T\},
\end{align}
that is, $y \in \Op T(x) \Leftrightarrow x \in \Op T^{-1}(y)$. Note that $\Op T^{-1}$ is always well-defined as a set-valued map with its value being $\emptyset$ when $y$ is not in the domain of $\Op T$. The inverse map $\Op T^{-1}$ is
single-valued (an ordinary function) if and only if $\Op T$ is bijective.

\begin{definition}
\label{Def:Operatorprops}
A set-valued operator $\Op T: \Spc H \to 2^{\Spc H}$, where $\Spc H$ is a Hilbert space equipped with the inner product $\langle \cdot,\cdot \rangle$, is said to be
\begin{enumerate}
\item {\em monotone} if $\langle y_2-y_1,x_2-x_1 \rangle\ge 0$ for all $(y_2,x_2), (y_1,x_1)\in {\rm graph}\Op T$; 
\item {\em strongly $\rho$-monotone} with $\rho>0$ if $\langle y_2-y_1,x_2-x_1 \rangle\ge \rho \|y_2-y_1\|^2$ for all $(y_2,x_2), (y_1,x_1)\in {\rm graph}\Op T$;
\item {\em weakly $\rho$-monotone} with $\rho>0$ if $\langle y_2-y_1,x_2-x_1 \rangle+\rho \|y_2-y_1\|^2\ge 0$ for all $(y_2,x_2), (y_1,x_1)\in {\rm graph}\Op T$;
\item \label{Item:Firm} {\em firmly non-expansive}  if
$\langle  y_2-y_1,  y_2-y_1\rangle \le \langle x_2-x_1,  y_2-y_1\rangle$
 for all $(y_2,x_2)$, $(y_1,x_1)\in {\rm graph}\Op T$ \cite{Bauschke2012firmly}.
 \end{enumerate}
\end{definition}
Note that the conditions in this definition are sometimes stated by replacing $y_2$ and $y_1$ by $\Op T(x_2)$ and $\Op T(x_1)$ with an implicit set-theoretic interpretation of the inequalities. For instance, the monotonicity condition may be written as  $\langle \Op T(x_2)-\Op T(x_1),x_2-x_1 \rangle\ge 0$, with the understanding that the left-hand side represents a subset of $\R$ that must be included in $(+\infty,0)$.

For any (proper) function $f: \R^d \to \overline{\R}$, the subdifferential $\partial f: \R^d \to 2^{\R^d}$ is defined as
\begin{align}
\partial f(\M x)=\{\M z \inR^d: f(\M y)\ge f(\M x)+\M z^\Top(\M y-\M x), \forall \M y \in\R^d\}.
\end{align}
While $\partial f(\M x)$ is specified as a set, it is typically a singleton. 
In particular, if $f$ is convex and differentiable at $\M x$, then $\partial f(\M x)=\{\M \nabla f(\M x)\}$ so that we can identify the subdifferential with the gradient of $f$. If, on the one hand, $f$ is nonconvex, then there usually exist values of $\M x$ such that
$\partial f(\M x)=\emptyset$. If, on the other hand, $f$ is convex, then $\partial f(\M x)$ is nonempty for every $\M x \in \R^d$, while the condition for optimality (Fermat's principle) is 
$$\M 0 \in \partial f(\M x_0) \Leftrightarrow f(\M x_0)=\inf_{\M x \R^d} f(\M x).$$

The prototypical example is $\partial |\cdot|(x)=\begin{cases} \{1\},& x>0\\
[-1,1], & x=0\\
\{-1\}, & x<0,
\end{cases}$
which returns the derivative of $|x|$ at the locations where it is well-defined and assigns the interval $[-1,1]$ at the origin where it is undefined. 

If $f: \R^d \to \R$ is either convex or $\rho$-strongly convex, then $\partial f$ is monotone ($\rho$-strongly monotone, respectively).

\bibliographystyle{ieeetr}
\bibliography{/Users/unser/MyDrive/Bibliography/Bibtex_files/Unser}
%
%
\end{document}